\documentclass{article}
\usepackage{iclr2021_conference,times}
\usepackage{booktabs} 

\usepackage{hyperref}

\usepackage{amsmath,amsfonts,bm}

\def\eqref#1{equation~\ref{#1}}

\def\1{\bm{1}}

\DeclareMathAlphabet{\mathsfit}{\encodingdefault}{\sfdefault}{m}{sl}
\SetMathAlphabet{\mathsfit}{bold}{\encodingdefault}{\sfdefault}{bx}{n}

\usepackage{url}

\usepackage{amsthm}
\usepackage{amssymb}

\usepackage{microtype}
\usepackage{graphicx}
\usepackage{subfigure}
\usepackage{amssymb}
\usepackage{mathtools}
\usepackage[pdf]{pstricks}
\usepackage{epsfig}
\usepackage{pst-grad} 
\usepackage{pst-plot} 
\usepackage{comment}
\usepackage{booktabs}
\usepackage{wrapfig}
\usepackage[nice]{nicefrac}

\newtheorem{innercustomgeneric}{\customgenericname}
\providecommand{\customgenericname}{}
\newcommand{\newcustomtheorem}[2]{%
  \newenvironment{#1}[1]
  {%
   \renewcommand\customgenericname{#2}%
   \renewcommand\theinnercustomgeneric{##1}%
   \innercustomgeneric
  }
  {\endinnercustomgeneric}
}

\newcustomtheorem{customdefine}{Definition}
\newcustomtheorem{customthm}{Theorem}
\newcustomtheorem{customcorollary}{Corollary}

\definecolor{azure}{rgb}{0.0, 0.5, 1.0}
\definecolor{greenpigment}{rgb}{0.0, 0.65, 0.31}
\definecolor{awesome}{rgb}{1.0, 0.13, 0.32}

\newcommand{\tasktext}[1]{\texttt{#1}}

\newtheorem{theorem}{Theorem}

\newtheorem{lemma}[theorem]{Lemma}
\newtheorem{defn}{Definition}

\title{Robust Constrained Reinforcement Learning for Continuous Control with Model Misspecification}

\author{Daniel J. Mankowitz\thanks{indicates equal contribution.}\\\texttt{dmankowitz@google.com} \And Dan A. Calian\footnotemark[1]\\\texttt{dancalian@google.com} \AND Rae Jeong \And Cosmin Paduraru \And Nicolas Heess \And Sumanth Dathathri \And Martin Riedmiller \And Timothy Mann \AND {\normalfont DeepMind}\\ London, UK}

\iclrfinalcopy 
\begin{document}

\maketitle

\begin{abstract}
Many real-world physical control systems are required to satisfy constraints upon deployment. Furthermore, real-world systems are often subject to effects such as non-stationarity, wear-and-tear, uncalibrated sensors and so on. Such effects effectively perturb the system dynamics and can cause a policy trained successfully in one domain to perform poorly when deployed to a perturbed version of the same domain. This can affect a policy's ability to maximize future rewards as well as the extent to which it satisfies constraints. We refer to this as constrained model misspecification. We present an algorithm that mitigates this form of misspecification, and showcase its performance in multiple \textit{simulated} Mujoco tasks from the Real World Reinforcement Learning (RWRL) suite.
\end{abstract}

\section{Introduction}

Reinforcement Learning (RL) has had a number of recent successes in various application domains which include computer games \citep{silver2017mastering, mnih2015human,tessler2017deep} and robotics \citep{abbas2018a}. As RL and deep learning continue  to scale, an increasing number of real-world applications may become viable candidates to take advantage of this technology. However, the application of RL to real-world systems is often associated with a number of challenges \citep{Dulac2019,dulac2020empirical}. We will focus on the following two:

\textbf{Challenge 1 - Constraint satisfaction}: One such challenge is that many real-world systems have constraints that need to be satisfied upon deployment (i.e., hard constraints); or at least the number of constraint violations as defined by the system need to be reduced as much as possible (i.e., soft-constraints). This is prevalent in applications ranging from physical control systems such as autonomous driving and robotics to user facing applications such as recommender systems. 

\textbf{Challenge 2 - Model Misspecification (MM)}: Many of these systems suffer from model misspecification. We refer to the situation in which an agent is trained in one environment but deployed in a different, perturbed version of the environment as an instance of \textit{model misspecification}. This may occur in many different applications and is well-motivated in the literature \citep{mankowitz2018learning, Mankowitz2019,derman2018soft,derman2019,iyengar2005robust,tamar2014scaling}. 

There has been much work on constrained optimization in the literature \citep{altman1999constrained,Tessler2018a,efroni2020explorationexploitation,achiam2017constrained,bohez2019value}.  However, to our knowledge,  the effect of model misspecification on an agent's ability to satisfy constraints at test time has not yet been investigated. 

\textbf{Constrained Model Misspecification (CMM)}:
We consider the scenario in which an agent is required to satisfy constraints at test time but is deployed in an environment that is different from its training environment (i.e., a perturbed version of the training environment). Deployment in a perturbed version of the environment may affect the return achieved by the agent as well as its ability to satisfy the constraints. We refer to this scenario as \textit{constrained model misspecification}.

This problem is prevalent in many real-world applications where constraints need to be satisfied but the environment is subject to state perturbations effects such as wear-and-tear, partial observability etc., the exact nature of which may be unknown at training time. Since such perturbations can significantly impact the agent's ability to satisfy the required constraints it is insufficient to simply ensure that constraints are satisfied in the unperturbed version of the environment. Instead, the presence of unknown environment variations needs to be factored into the training process. One area where such considerations are of particular practical relevance is sim2real transfer where the sim2real gap can make it hard to ensure that constraints will be satisfied on the real system \citep{andrychowicz2018learning,peng2018sim,wulfmeier2017mutual,rastogi2018sample,Christiano2016}. 


\textbf{Main Contributions}: In this paper, we aim to bridge the two worlds of model misspecification and constraint satisfaction. We present an RL objective that enables us to optimize a policy that aims to be robust to CMM. Our contributions are as follows: 
\begin{itemize}
    \item We introduce the Robust Return Robust Constraint (R3C) and Robust Constraint (RC) RL objectives that aim to mitigate CMM as defined above. This includes the definition of a Robust Constrained Markov Decision Process (RC-MDP).
    \item Define the corresponding R3C and RC value functions and Bellman operators. We also provide an argument showing that these Bellman operators converge to fixed points. These are implemented in the policy evaluation step of actor-critic R3C algorithms.
    \item Implement five different R3C and RC algorithmic variants on top of D4PG and DMPO, (state-of-the-art continuous control RL algorithms).
    \item Empirically demonstrate the superior performance of our algorithms, compared to various baselines, with respect to mitigating CMM. This is shown consistently across $6$ different Mujoco tasks from the Real-World RL (RWRL) suite\footnote{\url{https://github.com/google-research/realworldrl_suite}}. This includes an investigative study into the learning performance of the robust and non-robust variants respectively.
\end{itemize}
    
    

\section{Background}
\label{sec:background}

\subsection{Markov Decision Processes}
A \textbf{Robust Markov Decision Process (R-MDP)} is defined as a tuple $\langle S, A, R, \gamma, \mathcal{P} \rangle$ where $S$ is a finite set of states, $A$ is a finite set of actions, $R:S\times A \rightarrow \mathbb{R}$ is a bounded reward function and $\gamma \in [0, 1)$ is the discount factor; $\mathcal{P}(s,a) \subseteq \mathcal{M}(S)$ is an uncertainty set where $\mathcal{M}(S)$ is the set of probability measures over next states $s' \in S$. This is interpreted as an agent selecting a state and action pair, and the next state $s'$ is determined by a conditional measure $p(s' \vert s, a) \in \mathcal{P}(s,a)$ \citep{iyengar2005robust}. We want the agent to learn a policy $\pi:S \rightarrow A$, which is a mapping from states to actions that is robust with respect to this uncertainty set. For the purpose of this paper, we consider deterministic policies, but this can easily be extended to stochastic policies too. The robust value function $V^{\pi}:S \rightarrow \mathbb{R}$ for a policy $\pi$ is defined as $V^{\pi}(s) = \inf_{p \in \mathcal{P}(s,\pi(s))} V^{\pi, p}(s)$ where $V^{\pi,p}(s) = r(s,\pi(s)) + \gamma p(s'|s,\pi(s)) V^{\pi,p}(s')$. A rectangularity assumption on the uncertainty set \citep{iyengar2005robust} assumes that ``nature'' can choose a worst-case transition function independently for every state $s$ and action $a$. This means that during a trajectory, at each timestep, nature can choose any transition model from the uncertainty set to reduce the performance of the agent. A robust policy optimizes for the robust (worst-case) expected return objective: $J_{\text{R}}(\pi)=\inf_{p \in \mathcal{P}}\mathbb{E}^{p,\pi} [\sum_{t=0}^\infty \gamma^t r_t ]$. 

The robust value function can be expanded as ${V^{\pi}(s) = r(s,\pi(s)) + \gamma \inf_{p \in P(s,\pi(s))} \mathbb{E}^p [V^{\pi}(s')  | s, \pi(s)]}$. As in \citep{tamar2014scaling}, we  define an operator $\sigma^{inf}_{\mathcal{P}(s,a)}v : \mathbb{R}^{|S|}\rightarrow \mathbb{R}$ as $\sigma^{inf}_{\mathcal{P}(s,a)}v=\inf\{p^\top v | p \in \mathcal{P}(s,a)\}$. We can also define an operator for some policy $\pi$ as $\sigma^{inf}_{\pi}:\mathbb{R}^{|S|}\rightarrow \mathbb{R}^{|S|}$ where $\{ \sigma^{inf}_{\pi} v \}(s) = \sigma^{inf}_{\mathcal{P}(s,\pi(s))}v$. Then, we have defined the Robust Bellman operator as follows $T^{\pi}_{\mathcal{R}} V^\pi = r^\pi + \gamma \sigma^{\inf}_\pi V^\pi$. Both the robust Bellman operator $T_{\mathcal{R}}^\pi: \mathcal{R}^{|S|} \rightarrow \mathcal{R}^{|S|}$ for a fixed policy and the optimal robust Bellman operator $T^*_{\mathcal{R}}v(s) = \max_\pi T_{\mathcal{R}}^\pi v(s)$ have previously been shown to be contractions \citep{iyengar2005robust}.

A \textbf{Constrained Markov Decision Process (CMDP)} is an extension to an MDP and consists of the tuple $\langle S, A, P, R, C, \gamma \rangle$ where $S,A,R$ and $\gamma$ are defined as in the MDP above and $C:S\times A \rightarrow \mathbb{R}^K$ is a mapping from a state $s$ and action $a$ to a $K$ dimensional vector representing immediate costs relating to $K$ constraints. We use $K$=1 from here on in and therefore $C:S\times A \rightarrow \mathbb{R}$. We refer to the cost for a specific state action tuple $\langle s,a \rangle$ at time $t$ as $c_t(s,a)$. The solution to a CMDP is a policy $\pi:S\rightarrow \Delta_A$ that learns to maximize return and satisfy the constraints. The agent aims to learn a policy that maximizes the expected return objective $J_R^\pi=\mathbb{E}[\sum_{t=0}^\infty \gamma^{t}r_t]$ subject to $J_C^\pi=\mathbb{E}[\sum_{t=0}^\infty \gamma^{t}c_t] \leq \beta$ where $\beta$ is a pre-defined constraint threshold. A number of approaches~\citep{Tessler2018a,bohez2019value} optimize the Lagrange relaxation of this objective $\min_{\lambda \geq 0} \max_{\theta} J_R^\pi - \lambda (J_C^\pi- \beta)$ by optimizing the Lagrange multiplier $\lambda$ and the policy parameters $\theta$ using alternating optimization. We also define the constraint value function $V_C^{\pi,p}:S \rightarrow \mathbb{R}$ for a policy $\pi$ as in \citep{Tessler2018a} where $V_C^{\pi,p}(s) = c(s,\pi(s)) + \gamma p(s'|s,\pi(s)) V_C^{\pi,p}(s')$. 
 
 \subsection{Continuous Control RL Algorithms}

We address the CMM problem by modifying two well-known continuous control algorithms by having them optimize the RC and R3C objectives.

The first algorithm is \textbf{Distributed Distributional Deterministic Policy Gradient} (D4PG), which is a state-of-the-art actor-critic continuous control RL algorithm with a deterministic policy \citep{barth2018distributed}. It is an improvement to DDPG~\citep{lillicrap2015continuous} with a distributional critic that is learned similarly to distributional MPO.

 The second algorithm is \textbf{Maximum A-Posteriori Policy Optimization (MPO)}. This is a continuous control RL algorithm that performs policy iteration using an RL form of expectation maximization \citep{abbas2018a,abdolmaleki2018maximum}. We use the distributional-critic version in \citet{abdolmaleki2020distributional}, which we refer to as DMPO.

\section{Robust Constrained Optimization Framework}
We begin by defining a Robust Constrained MDP (RC-MDP). This combines an R-MDP and C-MDP to yield the tuple $\langle S, A, R, C, \gamma, \mathcal{P} \rangle$ where all of the variables in the tuple are defined in Section \ref{sec:background}. We next define two optimization objectives that optimize the RC-MDP. The first objective attempts to learn a policy that is robust with respect to the return as well as constraint satisfaction - Robust Return Robust Constrained (R3C) objective. The second objective is only robust with respect to constraint satisfaction - Robust Constrained (RC) objective. 

Prior to defining these objectives, we make use of the following definitions. 

\begin{defn}
The robust constrained value function $V_C^{\pi}:S \rightarrow \mathbb{R}$ for a policy $\pi$ is defined as $V_C^{\pi}(s) = \sup_{p \in \mathcal{P}(s,\pi(s))} V_C^{\pi, p}(s) = \sup_{p \in \mathcal{P}(s,\pi(s))} \mathbb{E}^{\pi,p}\biggl[ \sum_{t=0}^\infty \gamma^{t} c_t \biggr]$.
\end{defn}

This value function represents the worst-case sum of constraint penalties over the course of an episode with respect to the uncertainty set $\mathcal{P}(s,a)$. We further define several useful operators. The first operator $\sigma^{sup}_{\mathcal{P}(s,a)} : \mathbb{R}^{|S|}\rightarrow \mathbb{R}$ is defined as $\sigma^{sup}_{\mathcal{P}(s,a)}v=\sup\{p^\top v | p \in \mathcal{P}(s,a)\}$. In addition, we define an operator on vectors for some policy $\pi$ as $\sigma^{sup}_{\pi}:\mathbb{R}^{|S|}\rightarrow \mathbb{R}^{|S|}$ where $\{ \sigma^{sup}_{\pi} v \}(s) = \sigma^{sup}_{\mathcal{P}(s,\pi(s))}v$. Then, we can defined the Supremum Bellman operator $T_{sup}^\pi: \mathcal{R}^{|S|} \rightarrow \mathcal{R}^{|S|}$ as follows $T^{\pi}_{sup} V^\pi = r^\pi + \gamma \sigma^{\sup}_\pi V^\pi$. Note that this operator is a contraction since we get the same result if we replace $T_{inf}^\pi$ with $T_{sup}^\pi$ and replace $V$ with $-V$. An alternative derivation of the sup operator contraction is given in the Appendix, Section \ref{app:sup} for completeness.

\subsubsection{Robust Return Robust Constraint (R3C) Objective}
The R3C objective is defined as:

\begin{eqnarray}
    &\max_{\pi \in \Pi} \inf_{p\in P} \mathbb{E}^{p, \pi}\biggl[\sum_t \gamma^t r(s_t,a_t) \biggr]& \nonumber\\ 
    &\text{ s.t.} \sup_{p'\in \mathcal{P}}\mathbb{E}^{p',\pi}\biggl[\sum_t \gamma^t c(s_t,a_t) \biggr] \leq \beta &\nonumber \\
    \label{eq:maineq}
\end{eqnarray}

Note, a couple of interesting properties about this objective: (1) it focuses on being robust with respect to the return for a pre-defined set of perturbations; (2) the objective also attempts to be robust with respect to the worst case constraint value for the perturbation set. The Lagrangian relaxation form of \eqref{eq:maineq} is used to define an R3C value function.

\begin{defn}[R3C Value Function]
For a fixed $\lambda$, and using the above-mentioned rectangularity assumption \citep{iyengar2005robust}, the R3C value  function for a policy $\pi$ is defined as the concatenation of two value functions ${\mathbf{V}^\pi=f(\langle V^\pi, V^\pi_C \rangle) = V^\pi - \lambda V_C^{\pi}}$. This implies that we keep two separate estimates of $V^\pi$ and $V^\pi_C$ and combine them together to yield $\mathbf{V}^\pi$. The constraint threshold $\beta$ term offsets the value function, and has no effect on any policy improvement step\footnote{The $\beta$ term is only used in the Lagrange update in Lemma 1.}. As a result, the dependency on $\beta$ is dropped. 
\label{def:r3cval}
\end{defn}



The next step is to define the R3C Bellman operator. This is presented in Definition \ref{def:bellman}.

\begin{defn}[R3C Bellman operator]
The R3C Bellman operator is defined as two separate Bellman operators $T^{\pi}_{R3C}=\langle T^{\pi}_{inf}, T^{\pi}_{sup} \rangle$  where $T^{\pi}_{inf}$ is the robust Bellman operator \citep{iyengar2005robust} and $T^{\pi}_{sup}:\mathbb{R}^{|S|}\rightarrow \mathbb{R}^{|S|}$ is defined as the $\sup$ Bellman operator. Based on this definition, applying the R3C Bellman operator to $\mathbf{V} = \langle V, V_C \rangle$ involves applying each of the Bellman operators to their respective value functions. That is, $T^{\pi}_{R3C}\mathbf{V} = T^{\pi}_{inf}V - \lambda T^{\pi}_{sup}V_C$.
\label{def:bellman}
\end{defn}

\begin{theorem}
Given an arbitrary return value function $V:S\rightarrow \mathbb{R}$ and an arbitrary constraint value function $V_C:S\rightarrow \mathbb{R}$, the R3C Bellman operator $\mathcal{T}^{\pi}_{R3C}:\mathbb{R}^{|S|} \rightarrow \mathbb{R}^{|S|}$ when applied iteratively to $\mathbf{V}=\langle V, V_C \rangle$ converges to a fixed point. That is, ${T^\pi_{R3C} \mathbf{V}^{\pi} = \mathbf{V}^{\pi} =\langle V^\pi, V_C^\pi \rangle}$.
\end{theorem}

\begin{proof}
It has been previously shown that $T^{\pi}_{inf}$ is a contraction with respect to the max norm \citep{tamar2014scaling} and therefore converges to a fixed point. We also provided an argument whereby $T^{\pi}_{sup}$ is a contraction operator in the previous section as well as in Appendix, \ref{app:sup}. These Bellman operators individually ensure that the robust value function $V(s)$ and the constraint value function $V_C(s)$ converge to fixed points. Therefore, $\mathcal{T}^{\pi}_{R3C}\mathbf{V}$ also converges to a fixed point by construction. 
\end{proof}


As a result of the above argument, we know that we can apply the R3C Bellman operator in value iteration or policy iteration algorithms in the policy evaluation step. In practice we simultaneously learn estimates of both the robust value function $V^\pi(s)$ and the constraint value function $V_C^\pi(s)$ and combine these estimates to yield $\mathbf{V}^\pi(s)$. 

\textit{It is useful to note that this structure allows for a flexible framework which can define an objective using different combinations of $\sup$ and $\inf$ terms, yielding combined Bellman operators that are contraction mappings.} It is also possible to take the mean with respect to the uncertainty set yielding a soft-robust update \citep{derman2018soft,Mankowitz2019}. We do not derive all of the possible combinations of objectives in this paper, but note that the framework provides the flexibility to incorporate each of these objectives. We next define the RC objective.

\subsubsection{Robust Constrained (RC) Objective}
The RC objective focuses on being robust with respect to constraint satisfaction and is defined as:

\begin{eqnarray}
    &\max_{\pi \in \Pi} \mathbb{E}^{\pi,p}\biggl[\sum_t \gamma^t r(s_t,a_t) \biggr] & \nonumber\\  
    &\text{ s.t. } \sup_{p'\in \mathcal{P}}\mathbb{E}^{p',\pi}\biggl[\sum \gamma^t c(s_t,a_t) \biggr] \leq \beta & \nonumber\\
\end{eqnarray}

This objective differs from R3C in that it only focuses on being robust with respect to constraint satisfaction. This is especially useful in domains where perturbations are expected to have a significantly larger effect on constraint satisfaction than on the return. The corresponding value function is defined as in Definition \ref{def:r3cval}, except by replacing the robust value function in the concatenation with the expected value function $V^{\pi,p}$. The Bellman operator is also similar to Definition \ref{def:bellman}, where the expected return Bellman operator $T^\pi$ replaces $T^\pi_{\inf}$.






\subsection{Lagrange update}
For both objectives, we need to learn a policy that maximizes the return while satisfying the constraint. This involves performing alternating optimization on the Lagrange relaxation of the objective. The optimization procedure alternates between updating the actor/critic parameters and the Lagrange multiplier. For both objectives we have the same gradient update for the Lagrange multiplier:

\begin{lemma}[Lagrange derivative]
The gradient of the Lagrange multiplier $\lambda$ is ${    \frac{\partial}{\partial \lambda} f =  - \biggl( \sup_{p\in \mathcal{P}}\mathbb{E}^{p,\pi}\biggl[\sum_t \gamma^t c(s_t,a_t) \biggr] - \beta \biggr)}$, 
where $f$ is the R3C or RC objective loss.
\end{lemma}

This is an intuitive update in that the Lagrange multiplier is updated using the worst-case constraint violation estimate. If the worst-case estimate is larger than $\beta$, then the Lagrange multiplier is increased to add more weight to constraint satisfaction and vice versa. 

\section{Robust Constrained Policy Evaluation}
We now describe how the R3C Bellman operator can be used to perform policy evaluation. This policy evaluation step can be incorporated into any actor-critic algorithm. Instead of optimizing the regular distributional loss (e.g. the C51 loss in~\citet{bellemare2017distributional}), as regular D4PG and DMPO do, we optimize the worst-case distributional loss, which is the distance: $d \biggl(\mathbf{r_t} + \gamma \mathbf{V}_{\hat{\theta}}^{\pi_{k}}(s_{t+1}), \mathbf{V}_{\theta}^{\pi_{k}}(s_t) \biggr)$, 
%
%
where:
\begin{align}
\mathbf{V}_{\theta}^{\pi_{k}}(s_{t})=  \inf_{p \in \mathcal{P}(s_t, \pi(s_t))}& \biggl[V_{\theta}^{\pi_{k}}(s_{t+1} \sim p(\cdot | s_t, \pi(s_t)))  \biggr]\nonumber \\
                                                              -\lambda \sup_{p' \in \mathcal{P}(s_t, \pi(s_t))}& \biggl[V_{C,\theta}^{\pi_{k}}(s_{t+1} \sim p'(\cdot | s_t, \pi(s_t)))\biggr];
\end{align}

and $\mathcal{P}(s_t,\pi(s_t))$ is an uncertainty set for the current state $s_t$ and action $a_t$; $\pi_{k}$ is the current network's policy, and $\hat{\theta}$ denotes the target network parameters. The Bellman operators derived in the previous sections are repeatedly applied in this policy evaluation step depending on the optimization objective (e.g., R3C or RC). This would be utilized in the critic updates of D4PG and DMPO. Note that the action value function definition, $\mathbf{Q}_{\theta}^{\pi_{k}}(s_{t}, a_t)$, trivially follows.


\section{Experiments} 
We perform all experiments using domains from the Real-World Reinforcement Learning (RWRL) suite\footnote{\url{https://github.com/google-research/realworldrl_suite}}, namely {\footnotesize\texttt{cartpole:\{balance, swingup\}}}, {\footnotesize\texttt{walker:\{stand, walk, run\}}}, and {\footnotesize\texttt{quadruped:\{walk, run\}}}. We define a unique task in our experiments as a 6-tuple as seen in Table \ref{tab:canonical_def}. The parameters in the table correspond to the domain name, the variant for that domain (i.e.~RWRL task), the constraint being considered, the safety coefficient value, the constraint threshold and the type of robustness perturbation being applied to the dynamics respectively. In total, we have $6$ unique tasks on which we test our benchmark agents. The full list of tasks can be found in the Appendix, Table \ref{app:full_list}. The available constraints per domain can be found in the Appendix~\ref{app:constraint_defs}.

The baselines used in our paper can be seen in Table~\ref{tab:baselines}. C-ALG refers to the reward constrained, non-robust algorithms of the variants that we have adapted based on~\citep{Tessler2018a,Calian2020}; RC-ALG refers to the robust constraint algorithms corresponding to the Bellman operator $T^{\pi}_{RC}$; R3C-ALG refers to the robust return robust constrained algorithms corresponding to the Bellman operator $T^{\pi}_{R3C}$; SR3C-ALG refers to the soft robust (with respect to return) robust constraint algorithms and R-ALG refers to the robust return algorithms based on~\cite{Mankowitz2019} but with the addition of a constraint in the objective.


\begin{table*}[]
\centering
\begin{tabular}{cccccc}
\toprule
\textbf{Domain} & \textbf{Domain Variant} & \textbf{Constraint} & \textbf{Safety Coefficient} & \textbf{Threshold} & \textbf{Perturbation} \\
\midrule
Cartpole        & Swingup                 & Balance Velocity    & 0.3                         & 0.115              & Pole length\\          
\bottomrule
\end{tabular}
\caption{An example task definition.}
\label{tab:canonical_def}
\end{table*}

\begin{table*}[h!]
\centering
\begin{tabular}{lll}
\toprule
{Baseline Algorithm} & {Variants}                                      & {Baseline Description}           \\
\midrule
{C-ALG}              & C-D4PG, C-DMPO                               & Constraint aware, non-robust.\\ 
{RC-ALG}             & RC-D4PG, RC-DMPO                            & Robust constraint.\\ 
{R3C-ALG}            & R3C-D4PG, R3C-DMPO                         & Robust return robust constraint.\\ 
{R-ALG}              & R-D4PG, R-DMPO                               & Robust return, constraint aware.\\ 
{SR3C-ALG}           & SR3C-D4PG & Soft robust return, robust constraint.\\ 
\bottomrule
\end{tabular}
\caption{The baseline algorithms used in this work.}
\label{tab:baselines}
\end{table*}

\begin{figure*}[!h]
    \centering
        \includegraphics[width=0.43\textwidth]{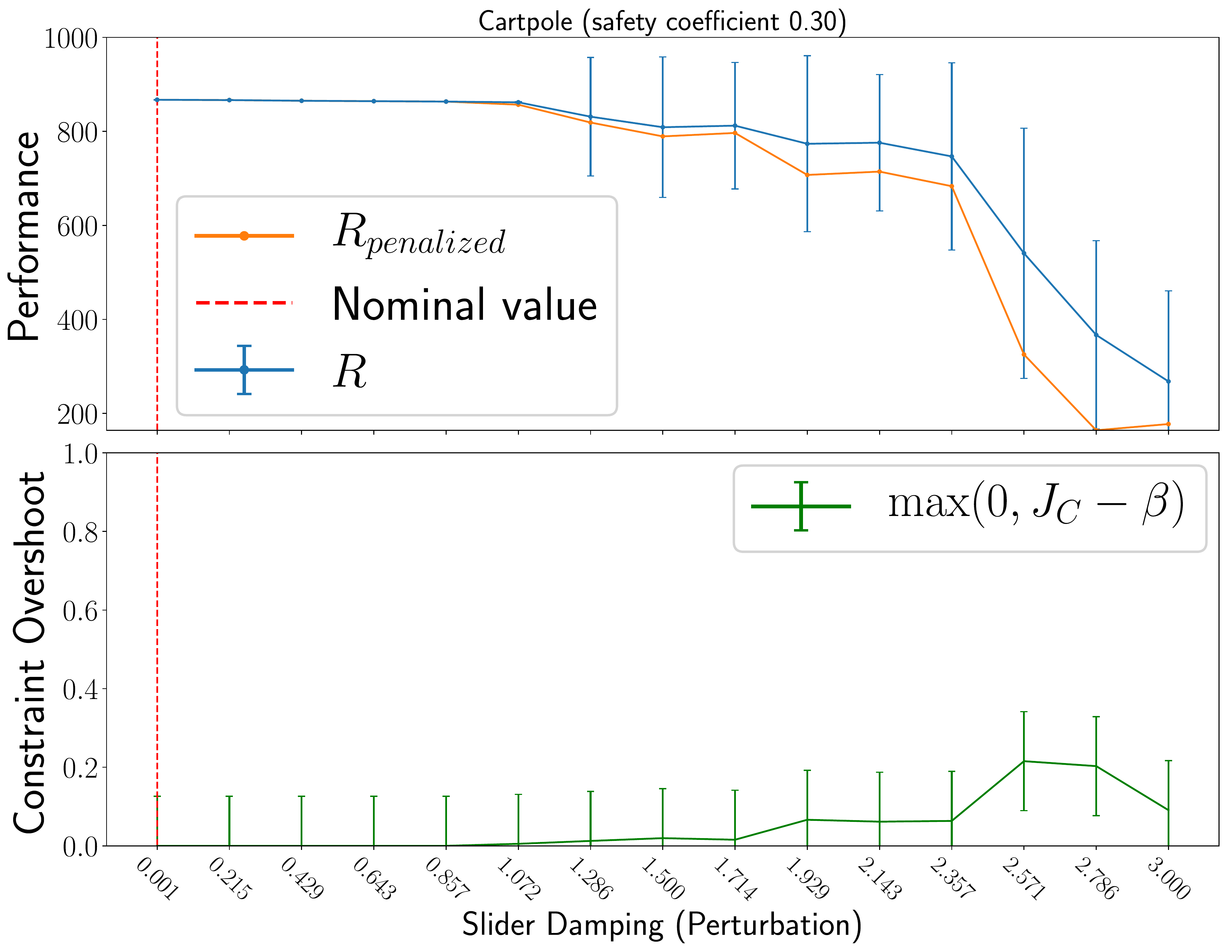}
        \includegraphics[width=0.43\textwidth]{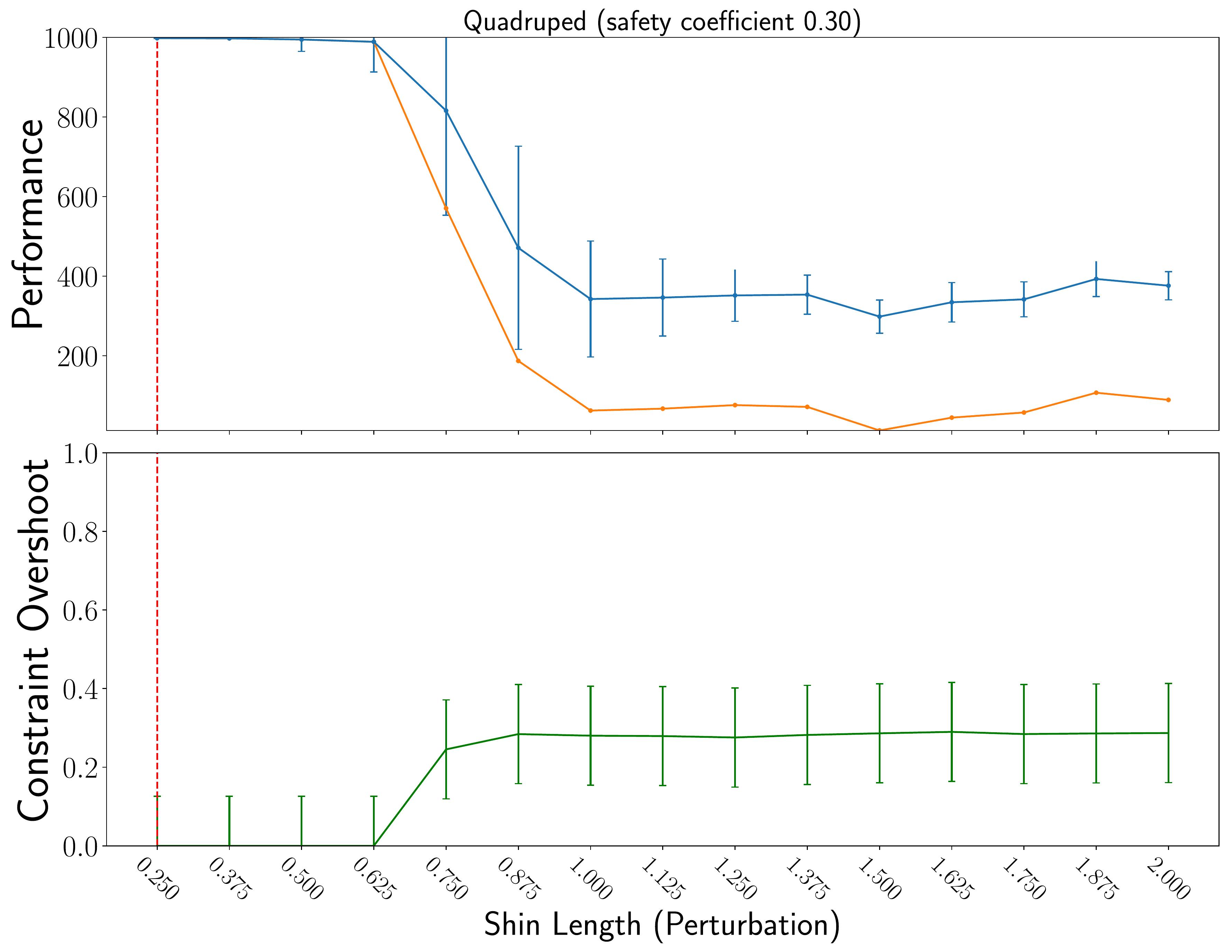}
        \includegraphics[width=0.43\textwidth]{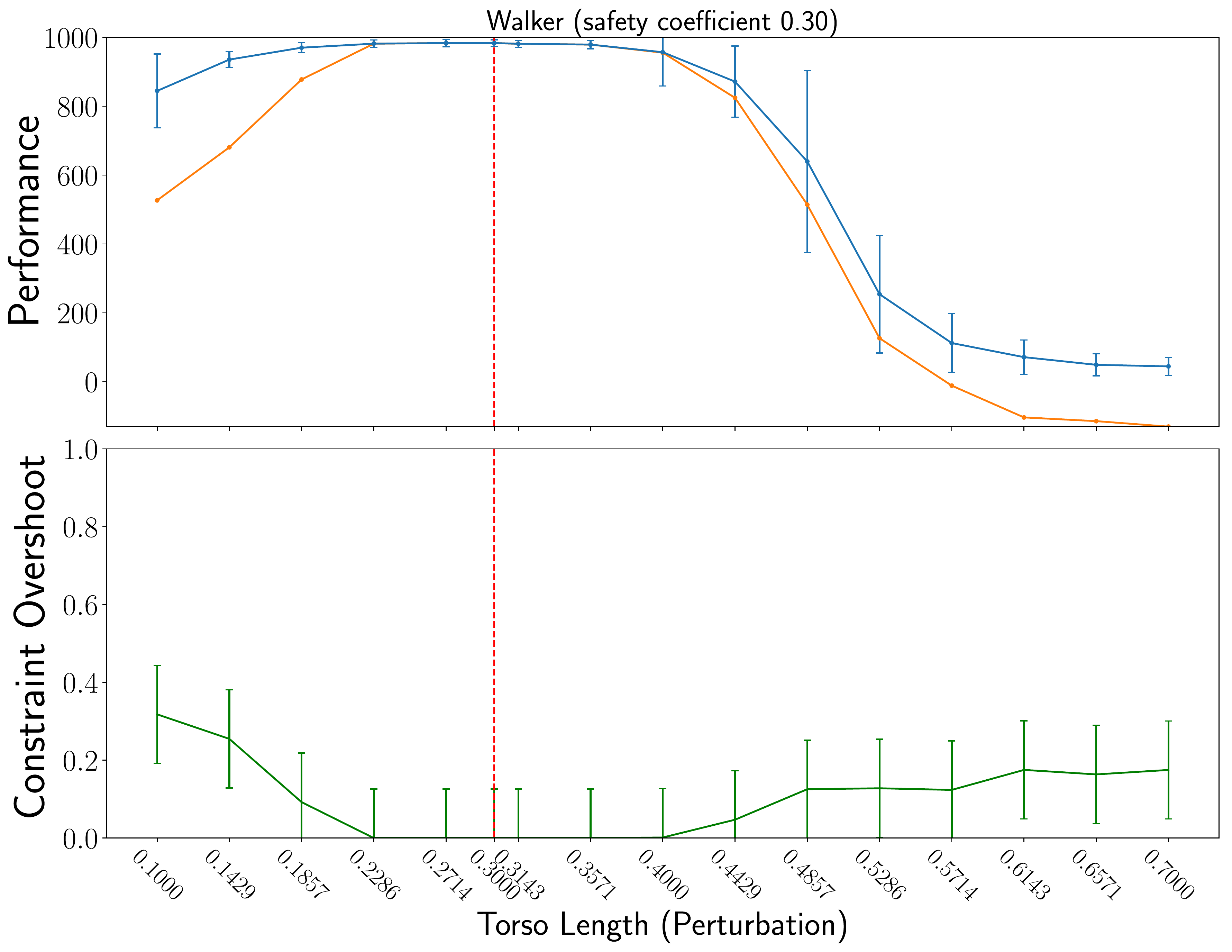}
    \caption{The effect on constraint satisfaction and return as perturbations are added to \tasktext{cartpole}, \tasktext{quadruped} and \tasktext{walker} for a fixed C-D4PG policy.}
    \label{fig:sensitivity_main_paper}    
\end{figure*}
\subsection{Experimental Setup}
For each task, the action and observation dimensions are shown in the Appendix, Table~\ref{tab:state_dimensions}. The length of an episode is $1000$ steps and the upper bound on reward is $1000$ \citep{Tassa2018}. All the network architectures are the same per algorithm and approximately the same across algorithms in terms of the layers and the number of parameters. A full list of all the network architecture details can be found in the Appendix, Table~\ref{tab:d4pg_hyperparameters}. All runs are averaged across $5$ seeds.

\textbf{Metrics}: We use three metrics to track overall performance, namely: \textbf{return} $R$, \textbf{overshoot} $\psi_{\beta,C}$ and \textbf{penalized return} $R_{penalized}$. The return is the sum of rewards the agent receives over the course of an episode. The constraint overshoot $\psi_{\beta,C}=\max(0, J_C^{\pi} - \beta)$ is defined as the clipped difference between the average costs over the course of an episode $J_C^{\pi}$ and the constraint threshold $\beta$. The penalized return is defined as $R_{penalized}=R - \bar{\lambda} \psi_{\beta,C}$ where $\bar{\lambda}=1000$ is an evaluation weight and equally trades off return with constraint overshoot $\psi_{\beta,C}$. 

\textbf{Constraint Experiment Setup}: The \textit{safety coefficient} is a flag in the RWRL suite~\citep{dulac2020empirical} that determines how easy/difficult it is in the environment to violate constraints. The safety coefficient values range from $0.0$ (easy to violate constraints) to $1.0$ (hard to violate constraints). As such we selected for each task (1) a safety coefficient of $0.3$; (2) a particular constraint supported by the RWRL suite and (3) a corresponding constraint threshold $\beta$, which ensures that the agent can find feasible solutions (i.e., satisfy constraints) and solve the task. 

\paragraph{Robustness Experimental Setup:} The robust/soft-robust agents (R3C and RC variants) are trained using a pre-defined uncertainty set consisting of $3$ task perturbations (this is based on the results from \cite{Mankowitz2019}). Each perturbation is a different instantiation of the Mujoco environment. The agent is then evaluated on a set of $9$ hold-out task perturbations ($10$ for \tasktext{quadruped}). For example, if the task is as defined in Table \ref{tab:canonical_def}, then the agent will have three pre-defined pole length perturbations for training, and evaluate on nine unseen pole lengths, while trying to satisfy the balance velocity constraint.

\textbf{Training Procedure}: All agents are always acting on the unperturbed environment. This corresponds to the default environment in the \mbox{dm\_control} suite~\citep{Tassa2018} and is referred to in the experiments as the \textit{nominal} environment. When the agent acts, it generates next state realizations for the nominal environment as well as each of the perturbed environments in the training uncertainty set to generate the tuple $\langle s,a,r, [s', s'_1, s'_2 \cdots s'_N] \rangle$ where N is the number of environments in the training uncertainty set and $s'_i$ is the next state realization corresponding to the $i^{th}$ perturbed training environment. Since the robustness update is incorporated into the policy evaluation stage of each algorithm, the critic loss which corresponds to the TD error in each case is modified as follows: when computing the target, the learner samples a tuple $\langle s,a,r, [s', s'_1, s'_2 \cdots s'_N] \rangle$ from the experience replay. The target action value function for each next state transition $[s', s'_1, s'_2 \cdots s'_N]$ is then computed by taking the $\inf$ (robust), average (soft-robust) or the nominal value (non-robust).  
In each case separate action-value functions are trained for the return $Q(s,a)$ and the constraint $Q_C(s,a)$. These value function estimates then individually return the $mean, \inf, \sup$ value, depending on the technique, and are combined to yield the target to compute $\mathbf{Q}(s,a)$. 

The chosen values of the uncertainty set and evaluation set for each domain can be found in Appendix, Table~\ref{app:uncertaintysets}. Note that it is common practice to manually select the pre-defined uncertainty set and the unseen test environments. Practitioners often have significant domain knowledge and can utilize this when choosing the uncertainty set \citep{derman2019,derman2018soft,Dicastro2012,mankowitz2018learning,tamar2014scaling}.
\begin{table*}[]
\centering
\small
\begin{tabular}{ccccccc}
\toprule
\textbf{Task} &\textbf{Domain} & \textbf{Domain Variant} & \textbf{Constraint} & \textbf{Safety Coefficient} & \textbf{Threshold} & \textbf{Perturbation} \\
\midrule
1&Cartpole        & Swingup                 & Balance Velocity    & 0.3                         & 0.115              & Slide damping\\
2 &Cartpole        & Swingup                 & Balance Velocity    & 0.3                         & 0.115              & Pole mass\\
3 & Walker        & Walk                 & Joint Velocity    & 0.3                         & 0.1              & Torso Length\\
4 & Walker        & Walk                 & Joint Velocity    & 0.3                         & 0.1              & Thigh Length\\
\bottomrule
\end{tabular}
\caption{The tasks presented in the experiments.}
\label{tab:main_task_defs}
\end{table*}%
%
\begin{table*}
\centering
\begin{tabular}{lllcc}
\toprule
 Base &  Algorithm &                 $R$ &  $R_{penalized}$ & $\max(0, J_C^\pi - \beta)$ \\
\midrule
 D4PG &     C-D4PG &  673.21 $\pm$ 93.04 &          491.450 &           0.18 $\pm$ 0.053 \\
  &     R-D4PG &  707.79 $\pm$ 65.00 &          542.022 &           0.17 $\pm$ 0.046 \\
  &   \textbf{R3C-D4PG} &  \textbf{734.45 $\pm$ 77.93} &          \textbf{635.246} &          \textbf{ 0.10 $\pm$ 0.049} \\
  &    RC-D4PG &  684.30 $\pm$ 83.69 &          578.598 &           0.11 $\pm$ 0.050 \\
  &  SR3C-D4PG &  723.11 $\pm$ 84.41 &          601.016 &           0.12 $\pm$ 0.038 \\
 \midrule
  DMPO &      C-DMPO &  598.75 $\pm$ 72.67 &          411.376 &           0.19 $\pm$ 0.049 \\
   &      R-DMPO &  686.13 $\pm$ 86.53 &          499.581 &           0.19 $\pm$ 0.036 \\
   &    \textbf{R3C-DMPO} &  \textbf{752.47 $\pm$ 57.10 }&          \textbf{652.969} &           \textbf{0.10 $\pm$ 0.040} \\
   &     RC-DMPO &  673.98 $\pm$ 80.91 &          555.809 &           0.12 $\pm$ 0.036 \\
\bottomrule
\end{tabular}
\caption{Performance metrics averaged over all holdout sets for all tasks.}
\label{tab:overall}
\end{table*}

\subsection{Main Results}
In the first sub-section we analyze the sensitivity of a fixed constrained policy (trained using C-D4PG) operating in perturbed versions of a given environment. This will help test the hypothesis that perturbing the environment does indeed have an effect on constraint satisfaction as well as on return. In the next sub-section we analyze the performance of the R3C and RC variants with respect to the baseline algorithms. We also investigate the learning performance of each baseline algorithm with respect to sample efficiency and the lagrangian parameter.

\subsubsection{Fixed Policy Sensitivity}
In order to validate the hypothesis that perturbing the environment affects constraint satisfaction and return, we trained a C-D4PG agent to satisfy constraints across $10$ different tasks. In each case, C-D4PG learns to solve the task and satisfy the constraints in expectation. We then perturbed each of the tasks with a supported perturbation and evaluated whether the constraint overshoot increases and the return decreases for the C-D4PG agent. Some example graphs are shown in Figure \ref{fig:sensitivity_main_paper} for the \tasktext{cartpole} (left), \tasktext{quadruped} (middle) and \tasktext{walker} (right) domains. The upper row of graphs contain the return performance (blue curve), the penalized return performance (orange curve) as a function of increased perturbations (x-axis). The vertical red dotted line indicates the nominal model on which the C-D4PG agent was trained. The lower row of graphs contain the constraint overshoot (green curve) as a function of varying perturbations. As seen in the figures, as perturbations increase across each dimension, both the return and penalized return degrades (top row) while the constraint overshoot (bottom row) increases. This provides useful evidence for our hypothesis that constraint satisfaction does indeed suffer as a result of perturbing the environment dynamics. This was consistent among many more settings. The full performance plots can be found in the Appendix, Figures \ref{fig:sensitivity_cartpole}, \ref{fig:sensitivity_quadruped} and \ref{fig:sensitivity_walker} for \tasktext{cartpole}, \tasktext{quadruped} and \tasktext{walker} respectively. 
\subsubsection{Overall performance}
We now compare C-ALG, RC-ALG, R3C-ALG, R-ALG and SR3C-ALG\footnote{We only ran the SR3C-D4PG variant to gain intuition as to soft-robust performance.} across $6$ tasks. The average performance across holdout sets and tasks is shown in Table \ref{tab:overall}. As seen in the table, the R3C-ALG variant outperforms all of the baselines in terms of return and constraint overshoot and therefore obtains the highest penalized return performance. Interestingly, the soft-robust variant yields competitive performance. 

We further analyze the results for three tasks using ALG=DMPO in Figure \ref{fig:robustnessdmpo} and ALG=D4PG in Figure~\ref{fig:robustnessd4pg} respectively. The three tasks are defined in Table \ref{tab:main_task_defs} where Task 1, Task 2 and Task 3 correspond to the left, middle and right columns respectively in each of the above-mentioned figures. Graphs of the additional tasks can be found in the Appendix, Figures \ref{fig:overall_results1} and \ref{fig:overall_results2}. Each graph contains, on the y-axis, the return $R$ (marked by the transparent colors) and the penalized return $R_{penalized}$ (marked by the dark colors superimposed on top of $R$). The x-axis consists of three holdout set environments in increasing order of difficulty from \emph{Holdout 0} to \emph{Holdout 8}. Holdout N corresponds to  perturbation element N for the corresponding task in the Appendix, Table \ref{app:uncertaintysets}. As can be seen for Tasks 1 and 2 (Figure \ref{fig:robustnessd4pg}, R3C-D4PG outperforms the baselines, especially as the perturbations get larger. This can be seen by observing that as the perturbations increase, the penalized return for these techniques is significantly higher than that of the baselines. This implies that the amount of constraint violations is significantly lower for these algorithms resulting in robust constraint satisfaction. Task 3 has similar improved performance over the baseline algorithms.  

\textbf{Soft-robustness}: As mentioned previously, we also ran experiments on a soft-robust variant SR3C-D4PG. This soft-robust objective is less conservative  (e.g., see \citet{derman2018soft}) as it takes the mean with respect to the uncertainty set instead of the infimum/supremum. 
As can be seen in Figure \ref{fig:robustnessd4pg} and Table \ref{tab:overall}, the performance is competitive with that of the robust variants (e.g., RC-D4PG and R3C-D4PG). However, this variant does tend to suffer if the perturbations become too large as can be seen in Figure \ref{fig:robustnessd4pg}(left) and (middle) respectively. This is consistent with previous work \cite{derman2018soft,Mankowitz2019}. If however, a less conservative approach is desirable, and the expected perturbations to the environment are not very large, the soft-robust approach might be a viable candidate algorithm.

\subsubsection{Investigative Studies}
\textbf{Learning performance}: We next perform an investigative study into the learning performance of ALG=DMPO on Task $4$ from Table \ref{tab:main_task_defs} (i.e., \tasktext{walker} with thigh length perturbations). 

In Figure~\ref{fig:learning_curves} we plot the learning curves for (1) episode return ($R$) and (2) constraint return ($J_C^\pi$) relative to $\beta$, the constraint threshold. Figures~\ref{fig:learning_curves} (a) and (c) correspond to the return and constraint satisfaction performance for each baseline trained on the nominal model. Figures~\ref{fig:learning_curves} (b) and (d) show the performance of the baselines on a holdout set. 

In this task, both R3C-DMPO and RC-DMPO manage to obtain policies which satisfy the constraint both on the unperturbed and the holdout environments. It is interesting to note, in this example, that the constraints are significantly below the performance threshold. This indicates that the robust variants may find an overly conservative solution. On the other hand, the algorithms that do not optimize for robust constraint satisfaction, namely C-DMPO and R-DMPO, have poor constraint satisfaction performance on the holdout set as seen in Figure~\ref{fig:learning_curves}.

\textbf{Lagrange multiplier learning performance}: We next investigate the Lagrange multiplier ($\lambda$) learning performance for each algorithm shown in Figure~\ref{fig:learning_curves} (e). 

In this example, we can see that both R-DMPO and C-DMPO converge to constraint satisfying policies on the nominal environment and this results in the Lagrange multiplier quickly converging to $0$ as seen in the figure. On the other hand, the robust constraint algorithms, R3C-DMPO and RC-DMPO have noisier lagrange multipliers. One might ask why the constraint is being satisfied in Figure~\ref{fig:learning_curves}(c), yet the Lagrange multiplier is non-zero. This is because the agent is acting only in the \textit{nominal} model and learns to satisfy constraints with respect to this model. This does not mean that the agent necessarily satisfies constraints in the other models in the uncertainty set, as it only has access to next state samples from these sets (rather than full trajectories). This explains why the Lagrange multiplier is non-zero. See Section \ref{app:investigate}, Figure \ref{fig:investigate} (e,g,h) for an example of this. However, as we have shown in the experiments, this still results in robust performance to unseen holdout sets.


%
\begin{figure*}[!h]
    \centering
        \includegraphics[width=0.49\textwidth]{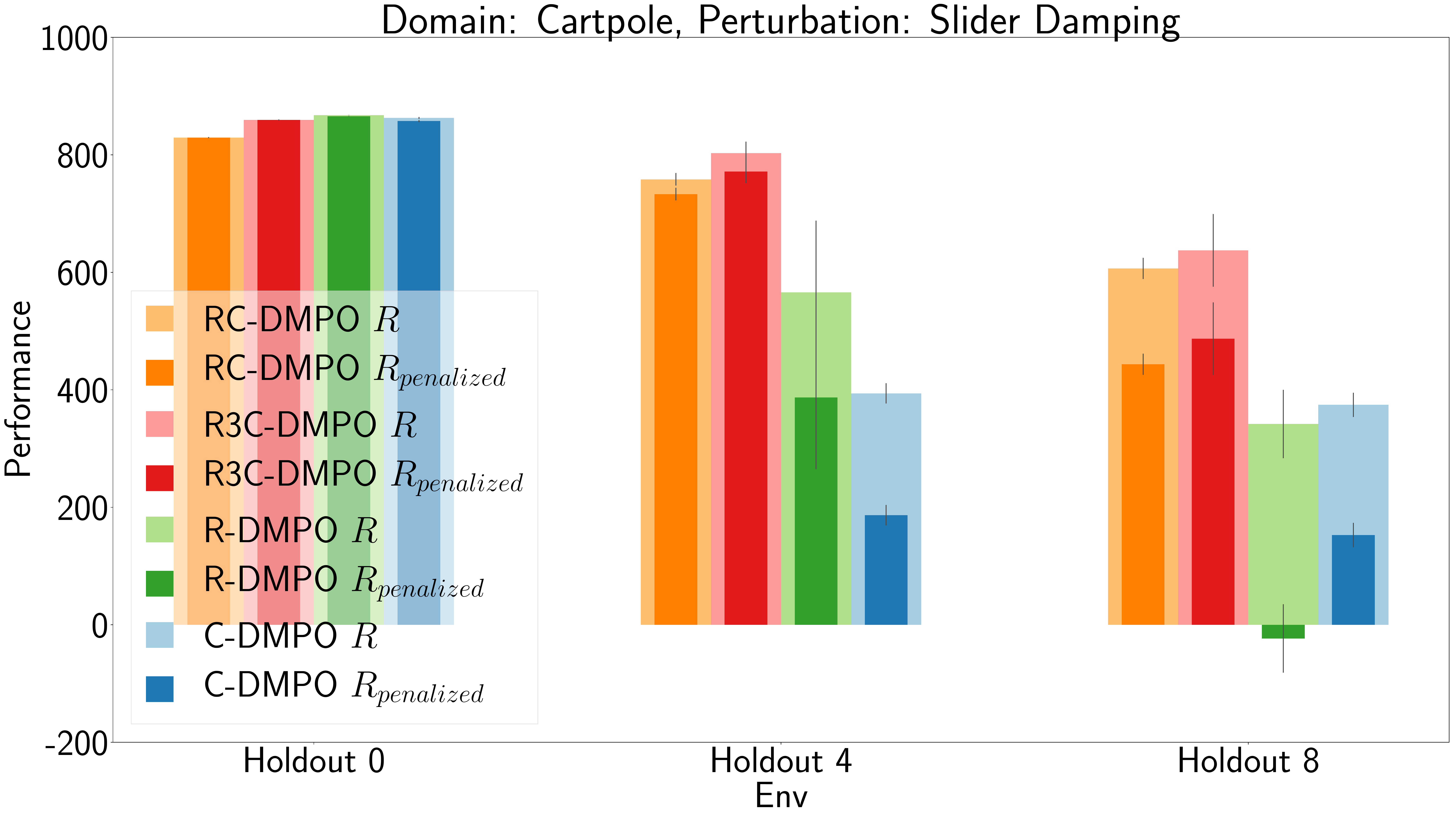}
        \includegraphics[width=0.49\textwidth]{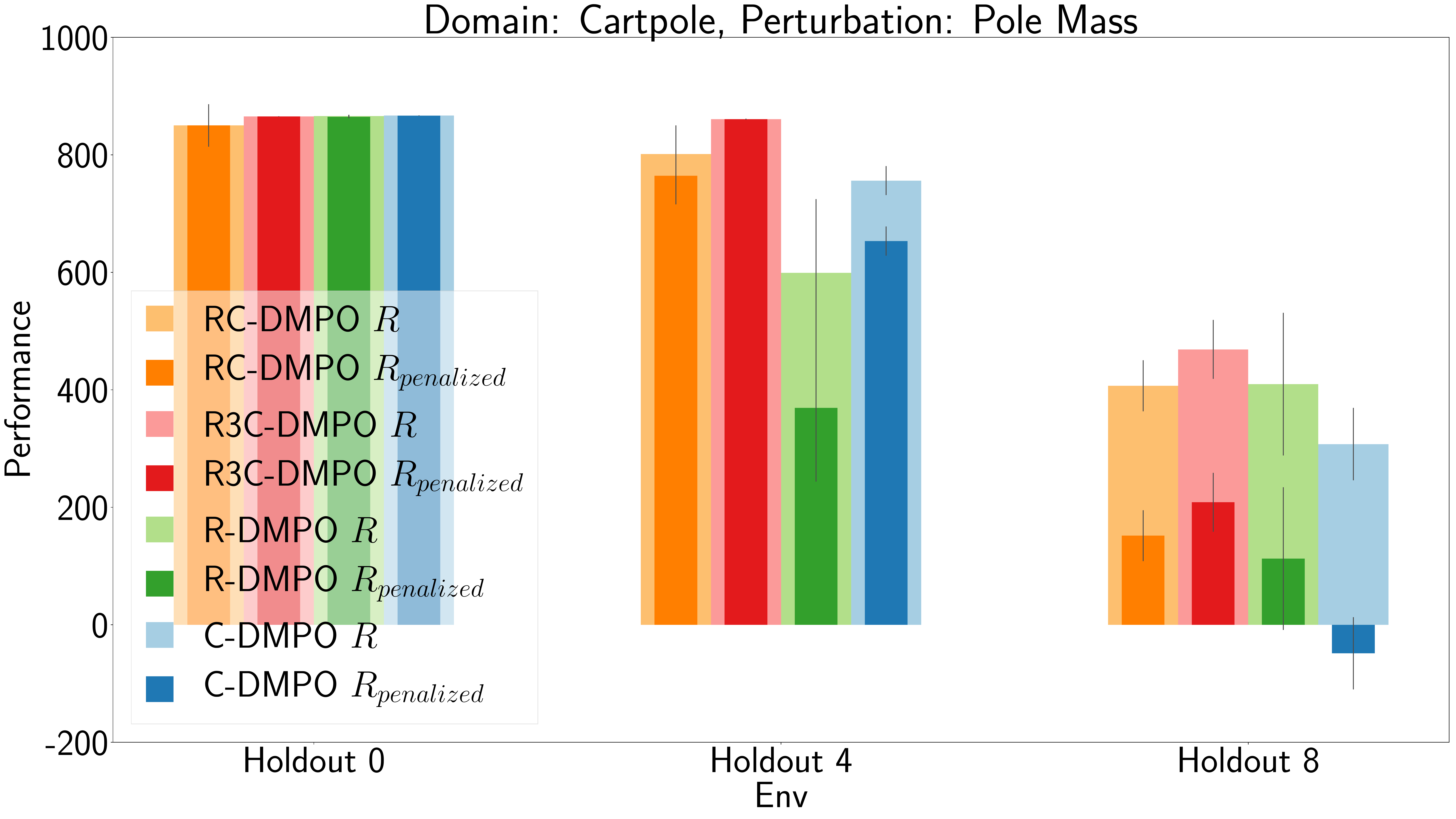}
        \includegraphics[width=0.49\textwidth]{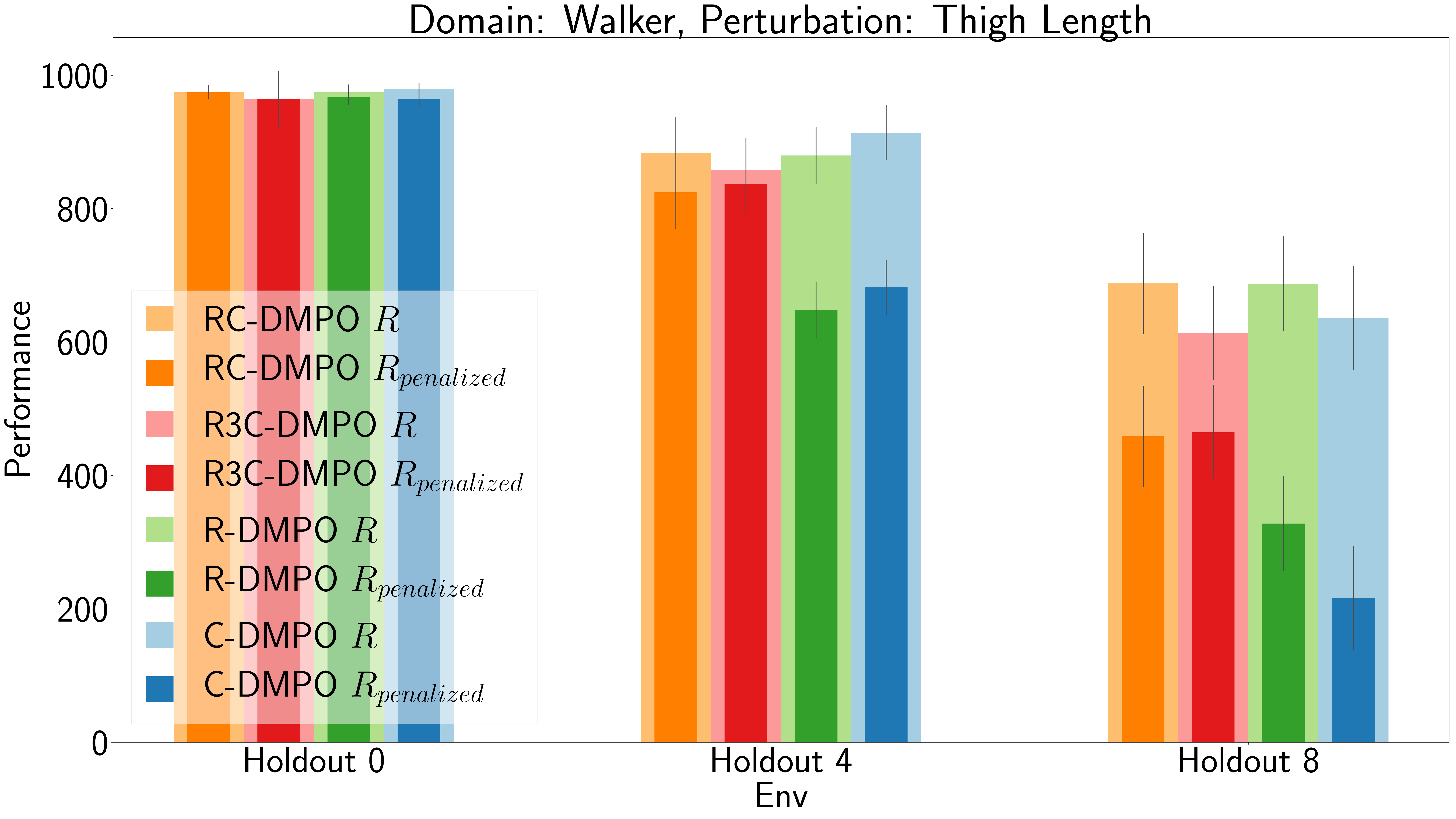}%
        \vspace{-0.25cm}
    \caption{The holdout set performance of the baseline algorithms on DMPO variants for Cartpole with slider damping and pole mass perturbations, and Walker with thigh length perturbations (bottom row).}
    \label{fig:robustnessdmpo}    
\end{figure*}

\begin{figure*}[!h]
    \centering
        \includegraphics[width=0.49\textwidth]{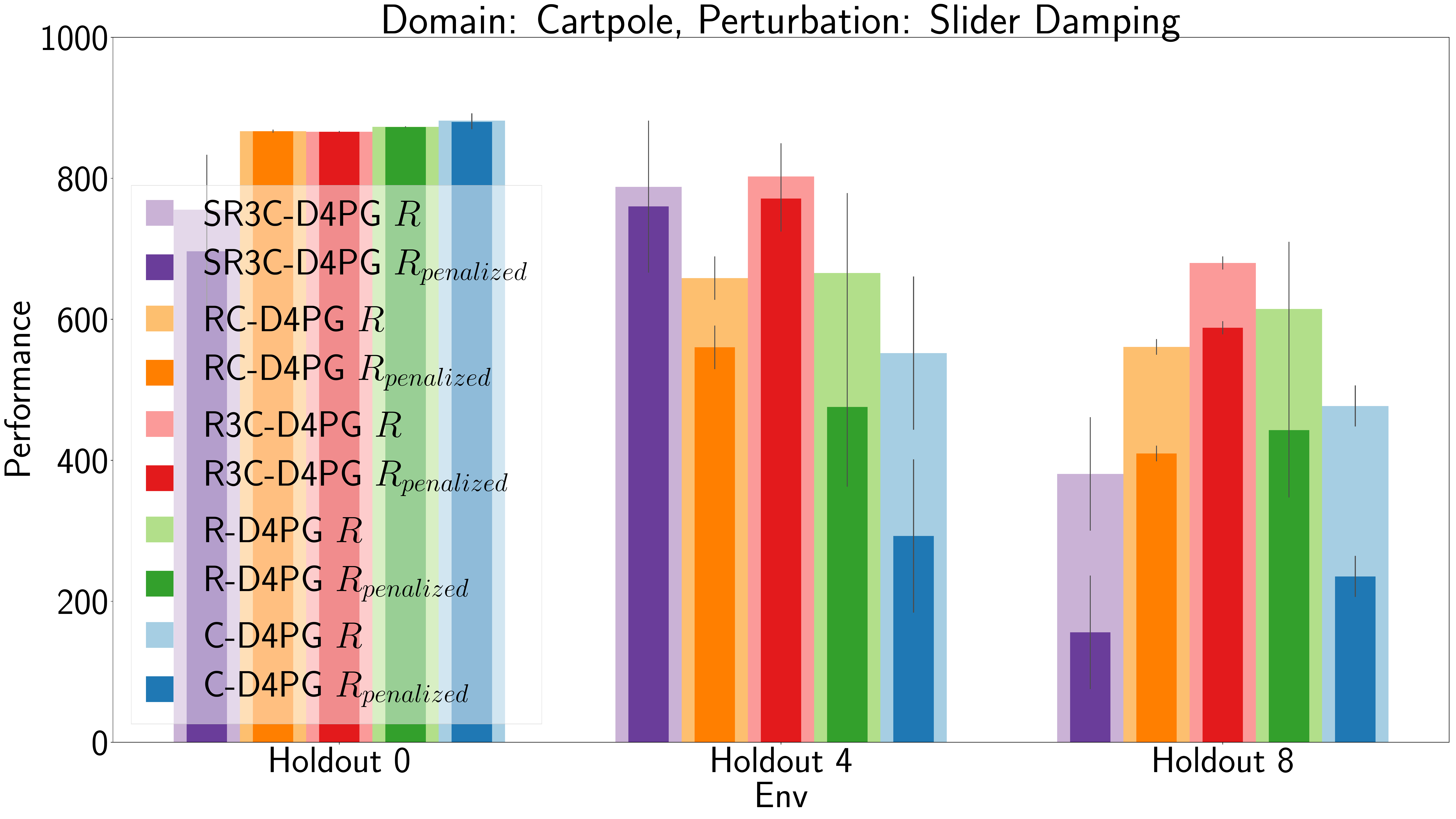}
        \includegraphics[width=0.49\textwidth]{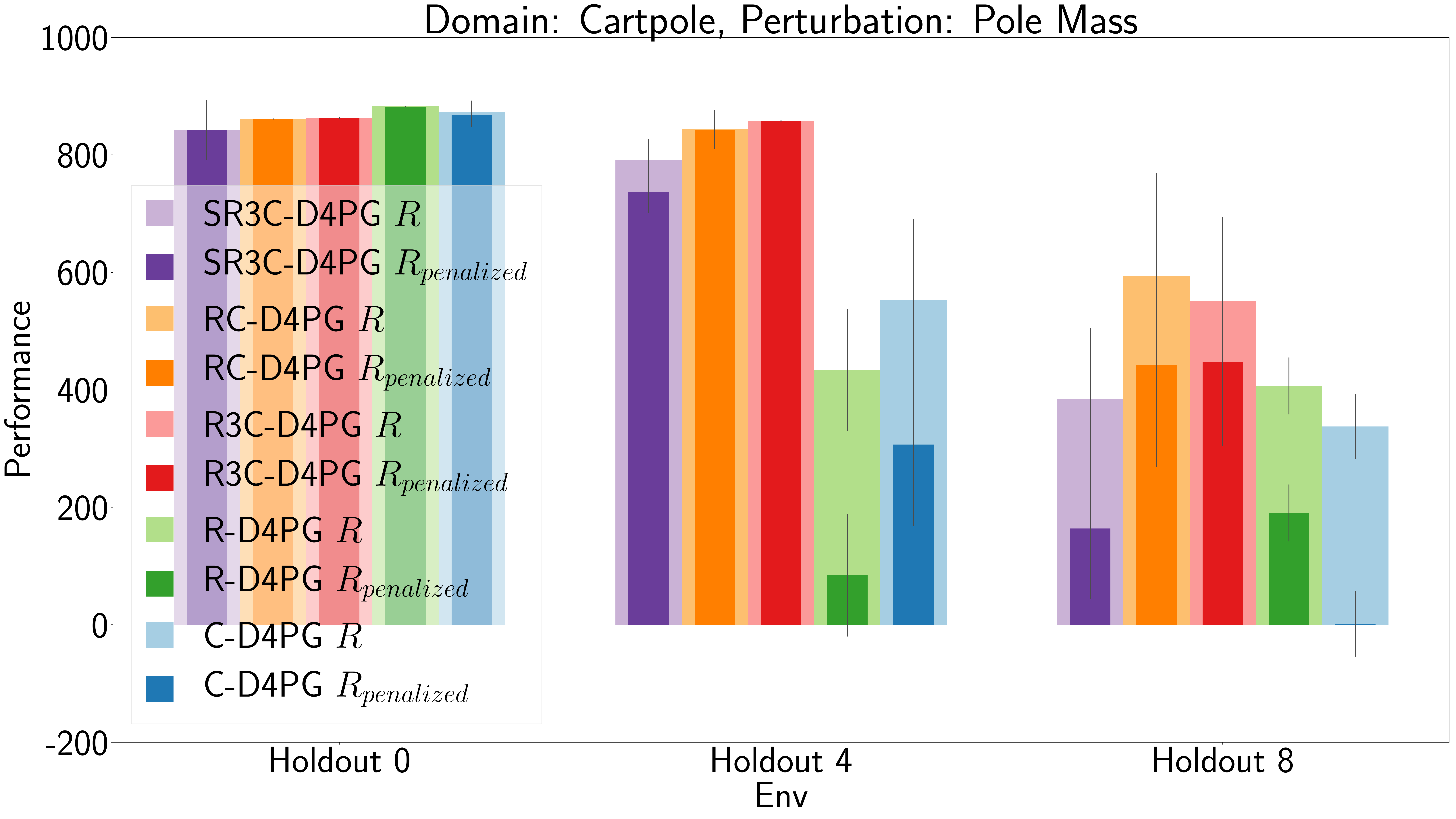}
        \includegraphics[width=0.49\textwidth]{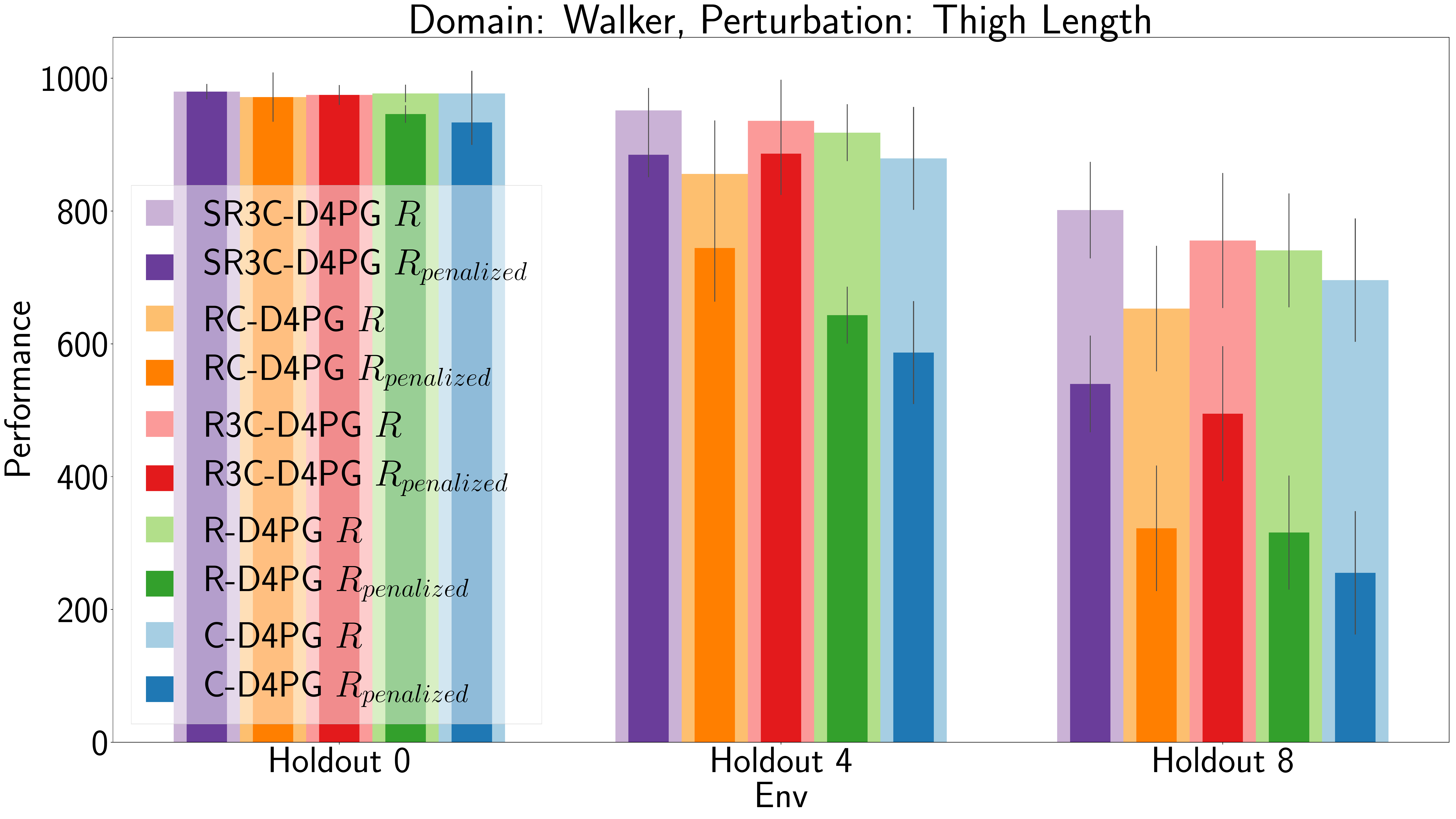}%
        \vspace{-0.25cm}
    \caption{The holdout set performance of the baseline algorithms on D4PG variants for Cartpole with slider damping and pole mass perturbations, and Walker with thigh length perturbations (bottom row).}
    \label{fig:robustnessd4pg}
\end{figure*}
%
%

\begin{figure}[!h]
    \centering
    \includegraphics[width=\columnwidth]{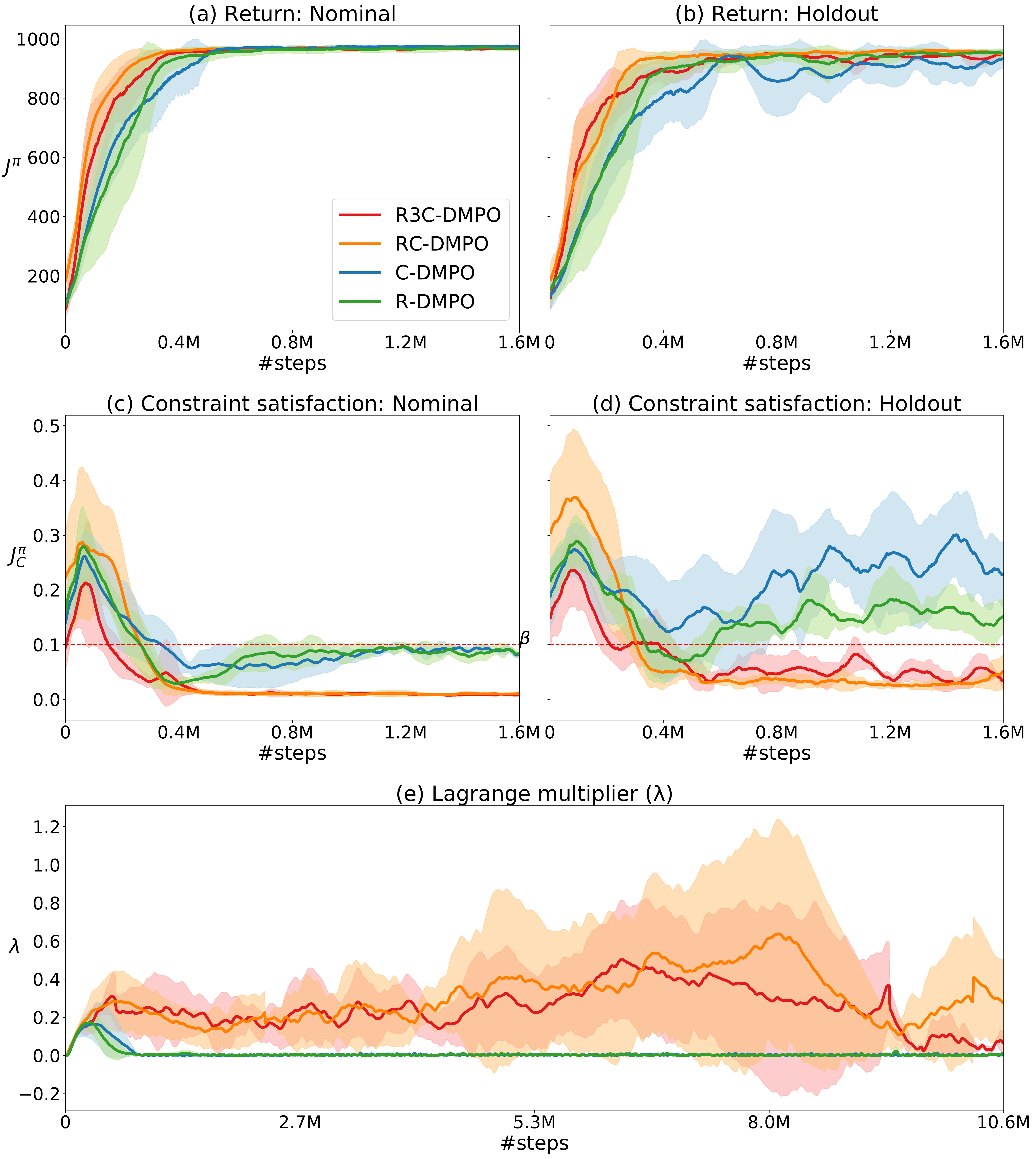}%
    \vspace{-0.25cm}
    \caption{Learning curves of the DMPO variants for Task 4 from Table~\ref{tab:main_task_defs} - the Walker domain with thigh length perturbations. This includes the episode return and constraint satisfaction performance (with respect to the threshold $\beta$) for the nominal model (a, c) and a holdout set (b, d) and (e) Lagrange learning performance. }
    \label{fig:learning_curves}
\end{figure}%

\section{Conclusion}
This paper simultaneously addresses robustness to constraint satisfaction and return with respect to state perturbations, two important challenges of real-world reinforcement learning which we collectively refer to as Constrained Model Misspecification (CMM). 

We present two RL objectives, R3C and RC, that yield robustness to constraints under the presence of state perturbations. We define R3C and RC Bellman operators to ensure that value-based RL algorithms will converge to a fixed point when optimizing these objectives. We then show that when incorporating this into the policy evaluation step of two well-known state-of-the-art continuous control RL algorithms the agent outperforms the baselines on $6$ Mujoco tasks. We also provide an investigative study into the learning performance of the robust and non-robust variants. We show that the robust variants may lead to an overly conservative solution with respect to constraint satisfaction.

In related work, \citet{everett2020certified} considers the problem of being verifiably robust to an adversary that can perturb the state $s'\in S$ to degrade performance as measured by a Q-function. \citet{2020scalable} consider the problem of learning agents (in deterministic environments with known dynamics) that satisfy constraints under perturbations to states $s'\in S$. In contrast, \eqref{eq:maineq} considers the general problem of learning agents that optimize for the return while satisfying constraints for a given RC-MDP. 
%
%
%
%

\bibliography{main}
\bibliographystyle{icml2021}

\onecolumn
\appendix
\section{Proofs}

\subsection{Lagrange Multiplier Gradient}
\label{app:lagrange}

\begin{proof}
\begin{eqnarray}
    \frac{\partial}{\partial \lambda} \inf_{p\in P} \mathbb{E}^{p, \pi}\biggl[\sum_t \gamma^t r(s_t,a_t) \biggr] - \lambda \biggl( \sup_{p\in \mathcal{P}}\mathbb{E}^{p,\pi}\biggl[\sum \gamma c(s_t,a_t) \biggr] - \beta \biggr) &&\\
    =  - \biggl( \sup_{p\in \mathcal{P}}\mathbb{E}^{p,\pi}\biggl[\sum \gamma c(s_t,a_t) \biggr] - \beta \biggr)
\end{eqnarray}

\begin{eqnarray}
    \min_{\lambda\geq 0} \max_{\pi \in \Pi} \inf_{p\in P} \mathbb{E}^{p, \pi}\biggl[\sum_t \gamma^t r(s_t,a_t) \biggr] - \lambda \biggl( \sup_{p\in \mathcal{P}}\mathbb{E}^{p,\pi}\biggl[\sum \gamma c(s_t,a_t) \biggr] - \beta \biggr) 
\end{eqnarray}

\begin{eqnarray}
    \frac{\partial}{\partial \lambda} \inf_{p\in P} \mathbb{E}^{p, \pi}\biggl[\sum_t \gamma^t r(s_t,a_t) \biggr] - \lambda \biggl( \sup_{p\in \mathcal{P}}\mathbb{E}^{p,\pi}\biggl[\sum \gamma c(s_t,a_t) \biggr] - \beta \biggr) &&\\
    =  - \biggl( \sup_{p\in \mathcal{P}}\mathbb{E}^{p,\pi}\biggl[\sum \gamma c(s_t,a_t) \biggr] - \beta \biggr)
\end{eqnarray}
\end{proof}

\subsection{The R3C value function}
\label{app:r3c_value}
\begin{eqnarray*}
    \mathbf{V}^{\pi}(s) = V^{\pi}(s) - \lambda V_C^{\pi}(s) &&
     \\ = \inf_{p \in P} \mathbb{E}^{p, \pi} \biggl[r(s,\pi(s)) + \gamma V^{\pi}(s') \biggr] - \lambda \biggl[ \sup_{p' \in P} \mathbb{E}^{p', \pi} \biggl[c(s,\pi(s)) + \gamma V^{\pi}_C(s') \biggr] \biggr] &&\\
    = \biggl[r(s,\pi(s)) + \gamma \inf_{p \in P} \mathbb{E}^{p, \pi} \biggl(V^{\pi}(s') \biggr) \biggr] - \lambda \biggl[  \biggl[c(s,\pi(s)) + \gamma \sup_{p' \in P} \mathbb{E}^{p', \pi} \biggl(V^{\pi}_C(s')\biggr) \biggr] \biggr] &&\\
    = r(s,\pi(s)) - \lambda c(s,\pi(s)) + \gamma \inf_{p \in P} \mathbb{E}^{p, \pi} \biggl(V^{\pi}(s') \biggr)  -   \lambda \gamma \sup_{p' \in P} \mathbb{E}^{p', \pi} \biggl(V^{\pi}_C(s')\biggr)  &&\\
    = r(s,\pi(s)) - \lambda c(s,\pi(s)) + \gamma \biggl[ \sigma^{inf}_{\mathcal{P}(s,\pi(s))} V^{\pi}  -   \lambda   \sigma^{sup}_{\mathcal{P}(s,\pi(s))} V^{\pi}_C \biggr]  &&\\
    = \mathbf{r}(s,\pi(s)) + \gamma \biggl[ \sigma^{inf}_{\mathcal{P}(s,\pi(s))} V^{\pi}  -   \lambda   \sigma^{sup}_{\mathcal{P}(s,\pi(s))} V^{\pi}_C \biggr] &&\\
\end{eqnarray*}

\subsection{Sup Bellman Operator}
\label{app:sup}

The R3C Bellman operator can be defined in terms of two separate Bellman operators: $T^{\pi}_{R3C}\mathbf{V}(s) = T^{\pi}_{inf}V(s) - \lambda T^{\pi}_{sup}V_C(s)$ where $T^{\pi}_{inf}:\mathbb{R}^{|S|}\rightarrow \mathbb{R}^{|S|}$ is the robust Bellman operator and $T^{\pi}_{sup}:\mathbb{R}^{|S|}\rightarrow \mathbb{R}^{|S|}$ is defined as the sup Bellman operator. These are defined as follows:

\begin{equation}
    T^{\pi}_{inf} V(s) =  r(s,\pi(s)) + \gamma \biggl[ \inf_{p \in P} \mathbb{E}^{p, \pi} \biggl(V(s') \biggr) \biggr]
\end{equation}

It has been previously shown that $T^{\pi}_{inf}$ is a contraction with respect to the max norm \citep{tamar2014scaling} and therefore converges to a fixed point. It remains to be shown that $T^{\pi}_{sup}$ is a contraction operator and that the R3C Bellman operator is a contraction operator.

\begin{equation}
    T^{\pi}_{sup} V_C(s) =  c(s,\pi(s)) + \gamma \biggl[\sup_{p \in P} \mathbb{E}^{P, \pi} \biggl(V_C(s')\biggr) \biggr] 
\end{equation}

\begin{theorem}[Sup Bellman operator contraction]
For two arbitrary value functions $U:S\rightarrow \mathbb{R}$ and $V:S\rightarrow \mathbb{R}$, we can show that the sup Bellman operator $\mathcal{T}^{\pi}_{sup}:\mathbb{R}^{|S|} \rightarrow \mathbb{R}^{|S|}$ is a contraction. That is 

\begin{equation}
\Vert \mathcal{T}^{\pi}_{sup} U - \mathcal{T}^{\pi}_{sup} V \Vert_\infty \leq \gamma \Vert U - V \Vert_\infty
\end{equation}
\end{theorem}

\begin{proof}

We follow the proofs from \citep{tamar2014scaling,iyengar2005robust}, Let $U,V \in \mathbb{R}^{|S|}$, and $s \in S$ an arbitrary state. Assume $\mathcal{T}^{\pi}_{sup} U(s) \geq \mathcal{T}^\pi_{sup} V(s)$. Let $\epsilon > 0$ be an arbitrary positive number. 

By the definition of the $\sup$ operator, there exists $p_{s} \in \mathcal{P}$ such that,

\begin{multline}
    c(s,a) + \gamma \sup_{p \in \mathcal{P}} \mathbb{E}_{a \sim \pi(\cdot |s)}U(s') 
    <  c(s, a) +  \gamma \mathbb{E}_{s' \sim p_{s}(\cdot|s,a)} U(s') + \epsilon  \\  
\end{multline}

In addition, we have by definition that:
\begin{multline}
    c(s,a) + \gamma \sup_{p \in \mathcal{P}} \mathbb{E}_{a \sim \pi(\cdot |s)}V(s') 
    >  c(s, a) +  \gamma \mathbb{E}_{s' \sim p_{s}(\cdot|s,a)} V(s')    \\
\end{multline}

Thus, we have,

\begin{equation}
\begin{split}
0 &\leq \mathcal{T}^{\pi}_{sup} U(s) - \mathcal{T}^{\pi}_{sup} V(s)\\
&<   \mathbb{E}_{a \sim \pi(\cdot|s)} [c(s,a) + \gamma \mathbb{E}_{s' \sim p_{s}(\cdot|s,a)} U(s')] + \epsilon - \biggl[\mathbb{E}_{a \sim \pi(\cdot|s)} [c(s,a) + \gamma \mathbb{E}_{s' \sim p_{s}(\cdot|s,a)} V(s')] \biggr]\\
&= \mathbb{E}_{a \sim \pi(\cdot |s), s' \sim p_s(\cdot|s,a)} [ \gamma U(s')]  - \mathbb{E}_{a \sim \pi(\cdot |s), s' \sim p_s(\cdot|s,a)} [ \gamma V(s')] + \epsilon\\
&\leq \gamma \vert U(s') - V(s') \vert + \epsilon\\
&  \leq \gamma \Vert U - V \Vert_\infty  + \epsilon\\
\end{split}
\end{equation}

Applying a similar argument for the case: $\mathcal{T}^{\pi} \tilde{U}(s) \leq \mathcal{T}^\pi \tilde{V}(s)$ results in:

\begin{equation}
\vert \mathcal{T}^{\pi} U(s) - \mathcal{T}^{\pi} V(s) \vert_\infty \leq \gamma \Vert U - V \Vert_\infty + \epsilon
\end{equation}

Since $\epsilon$ is arbitrary, we establish the result:

\begin{equation}
\Vert \mathcal{T}^{\pi} U - \mathcal{T}^{\pi} V \Vert_\infty \leq \gamma \Vert U - V \Vert_\infty
\end{equation}

\end{proof}

The previous two contraction mappings ensure that $V(s)$ and $V_C(s)$ converge to fixed points. However, in order to prove that the combination converges to a unique fixed point, we need to prove that $\mathcal{T}^{\pi}_{R3C}$ is a contraction mapping.

\label{app:sup}

\section{Experiments}

\subsection{Constraint Definitions}
\label{app:constraint_defs}
The per-domain safety constraints of the Real-World Reinforcement Learning (RWRL) suite that we use in the paper are given in Table~\ref{tab:safe_envs_full}. 

\begin{table}
\small
\begin{center}
\begin{tabular}{p{4cm} p{3.5cm}}
 \multicolumn{2}{l}{\textbf{Cartpole} variables: $x, \theta$} \\
 \hline
 {Type} & {Constraint} \\
 \hline
 \texttt{slider\_pos} & 
  $x_l < x < x_r$ \\
 \texttt{slider\_accel}
 &$ \ddot{x} < A_\textit{max}$ \\
 \texttt{balance\_velocity}*
 & $ \left| \theta \right| > \theta_L \vee \dot{\theta} < \dot{\theta}_V $ \\
 \hline
\end{tabular}
\\
\begin{tabular}[t]{p{4cm} p{3.5cm}}
\multicolumn{2}{l}{} \\
\multicolumn{2}{l}{\textbf{Walker} variables: $\bm{\theta}, \bm{u}, \bm{F}$} \\
\hline
 {Type} & {Constraint} \\
 \hline
 \texttt{joint\_angle}
 & $\bm{\theta}_{L} < \bm{\theta} < \bm{\theta}_{U}$ \\
 \texttt{joint\_velocity}*
 & $ \max_i \left| \dot{\bm{\theta}_i} \right| < L_{\dot{\theta}} $ \\
 \texttt{dangerous\_fall}
 & $0 < (\bm{u}_z \cdot \bm{x})$ \\
 \texttt{torso\_upright}
  &  $0 < \bm{u}_z$\\
 \hline
\end{tabular}
\\
\begin{tabular}[t]{p{4cm} p{3.5cm}}
\multicolumn{2}{l}{} \\
\multicolumn{2}{l}{\textbf{Quadruped} variables: $\bm{\theta}, \bm{u}, \bm{F}$} \\
\hline
 {Type} & {Constraint} \\
 \hline
  \texttt{joint\_angle}*
 & $\theta_{L, i} < \bm{\theta}_i < \theta_{U,i}$ \\
 \texttt{joint\_velocity}
 & $ \max_i \left| \dot{\bm{\theta}_i} \right| < L_{\dot{\theta}} $ \\
 \texttt{upright}
 &  $0 < \bm{u}_z$\\
\texttt{foot\_force}
 &$ \bm{F}_{\textit{EE}} < F_\textit{max} $ \\
\hline
\end{tabular}
\\
\end{center}
\caption{Safety constraints available for each RWRL suite domain; the constraints we use in this paper are indicated by an asterisk (*).}
\label{tab:safe_envs_full}
\end{table}


\subsection{Hyperparameters}

The hyperparameters used for all variants of D4PG can be found in Table \ref{tab:d4pg_hyperparameters}. The hyperparameters for the DMPO variants can be found in Table \ref{table:mpo_hyperparameters}.

\begin{table}[htpb]
\begin{center}
 \begin{tabular}{l  c} 
 \toprule
 D4PG Hyperparameters & Value \\
 \midrule
 Policy net & 256-256-256 \\ 
 $\sigma$ (exploration noise) & 0.1 \\
 Critic net & 512-512-256 \\ 
 Critic num. atoms & 51\\
 Critic vmin & -150 \\
 Critic vmax & 150\\ 
 N-step transition & 5\\
 Discount factor ($\gamma$) & 0.99 \\
 Policy and critic opt. learning rate & 0.0001 \\
 Replay buffer size & 1000000 \\
 Target network update period & 100\\
 Batch size & 256\\
 Activation function & elu\\
 Layer norm on first layer & Yes\\
 Tanh on output of layer norm & Yes\\
 \bottomrule
\end{tabular}
\end{center}
\caption{Hyperparameters for all variants of D4PG.\label{tab:d4pg_hyperparameters}}
\end{table}

\begin{table}[t]
\begin{center}
 \begin{tabular}{l c} 
 \toprule
 DMPO Hyperparameters & Value \\
 \midrule
 Policy net & 256-256-256 \\ 
 Number of actions sampled per state& 20\\
 Q function net & 512-512-512 \\
 Critic num. atoms & 51\\
 Critic vmin & -150 \\
 Critic vmax & 150\\ 
 $\epsilon$ & 0.1 \\
 $\epsilon_{\mu}$ & $1e-02$ \\
 $\epsilon_{\Sigma}$ & $1e-06$\\
 Discount factor ($\gamma$) & 0.99 \\
 Adam learning rate & $1e-04$ \\
 Replay buffer size & 1000000 \\
 Target network update period & 100\\
 Batch size & 256\\
 Activation function & elu\\
 Layer norm on first layer & Yes\\
 Tanh on output of layer norm & Yes\\
 Tanh on Gaussian mean & No \\
 Min variance & Zero \\
 Max variance & unbounded \\
\bottomrule
\end{tabular}
\end{center}
\caption{Hyperparameters for all variants of DMPO.}
\label{table:mpo_hyperparameters}
\end{table}

\begin{table}[htpb]
\begin{center}
\begin{tabular}{l c c}
\toprule
{RWRL Domain: Task}             & {Observation Dimension} & {Action Dimension} \\ \midrule
{Cartpole: Swingup} & 5                        & 1                         \\ 
{Walker: Walk}      & 18                       & 6                         \\ 
{Quadruped: Walk}   & 78                       & 12                          \\ 
\bottomrule
\end{tabular}
\end{center}
\caption{Observation and action dimension for each RWRL domain: task pair.\label{tab:state_dimensions}}
\end{table}

\begin{table}[]
\centering
\begin{tabular}{llllll}
\toprule
Domain & Variant & Constraint & Safety coeff & Threshold & Perturbation Type \\ 
\midrule
{Cartpole} & Swingup & Balance Velocity & 0.3  & 0.115  & Slider Damping \\
{Cartpole} & Swingup & Balance Velocity & 0.3  & 0.115  & Joint Damping \\
{Cartpole} & Swingup & Balance Velocity & 0.3  & 0.115  & Pole Mass  \\
\midrule
{Quadruped} & Walk  & Joint Angle & 0.3  & 0.7  & Shin Length \\
\midrule
{Walker} & Walk  & Joint Velocity & 0.3  & 0.1  & Thigh Length \\
{Walker} & Walk  & Joint Velocity & 0.3  & 0.1  & Torso Length \\
\bottomrule
\end{tabular}
\caption{The full list of the tasks we defined from the Real World RL Suite.}
\label{app:full_list}
\end{table}

\begin{table}[]
\centering
\tiny
\begin{tabular}{lllll}
\toprule
{Domain} & {Perturbation Type} & {Nom.\ Val.} & {Training Uncertainty Set} & {Holdout Set}                            \\
\midrule
Cartpole        & Joint Damping                                   & 0.0                                   & {[}0.0, 0.005, 0.01{]}                                 & {[}0.0025, 0.007,  0.008, 0.009, 0.015, 0.02, 0.025, 0.03, 0.035{]} \\
                & Slider Damping                                  & 0.001                                 & {[}0.001, 1.7, 1.9{]}                                  & {[}1.0, 1.4, 1.6, 1.8, 2.1, 2.3, 2.4, 2.5, 2.6{]}                 \\
                & Pole Mass                                       & 0.1                                   & {[}0.1, 0.2, 0.5{]}                                    & {[}0.05, 0.15, 0.25, 0.35, 0.45, 0.55, 0.6, 0.7, 0.8{]}             \\
\midrule
Quadruped       & Shin Length                                     & 0.25                                  & {[}0.25, 0.625, 0.70{]}                                & {[}0.85, 0.88, 0.92, 0.96, 1.0, 1.04, 1.08, 1.12, 1.16, 1.2{]}           \\
\midrule
Walker          & Thigh Length                                    & 0.225                                 & {[}0.225, 0.20, 0.17{]}                                & {[}0.21, 0.19, 0.185, 0.175,0.165, 0.155, 0.152, 0.15, 0.148{]}   \\
& Torso Length                                    & 0.3                                   & {[}0.3, 0.32, 0.34{]}                                  & {[}0.42, 0.43, 0.45, 0.47, 0.49, 0.51, 0.53. 0.55, 0.57{]}          \\
\bottomrule
\end{tabular}
\caption{Final experiment parameters.}
\label{app:uncertaintysets}
\end{table}

\subsection{Sensitivity to a fixed policy}

Figures \ref{fig:sensitivity_cartpole}, \ref{fig:sensitivity_quadruped} and \ref{fig:sensitivity_walker} show the sensitivity to perturbations of a fixed RC-D4PG policy on Cartpole, Quadruped and Walker respectively.

\begin{figure}[!h]
    \centering
        \includegraphics[width=0.48\textwidth]{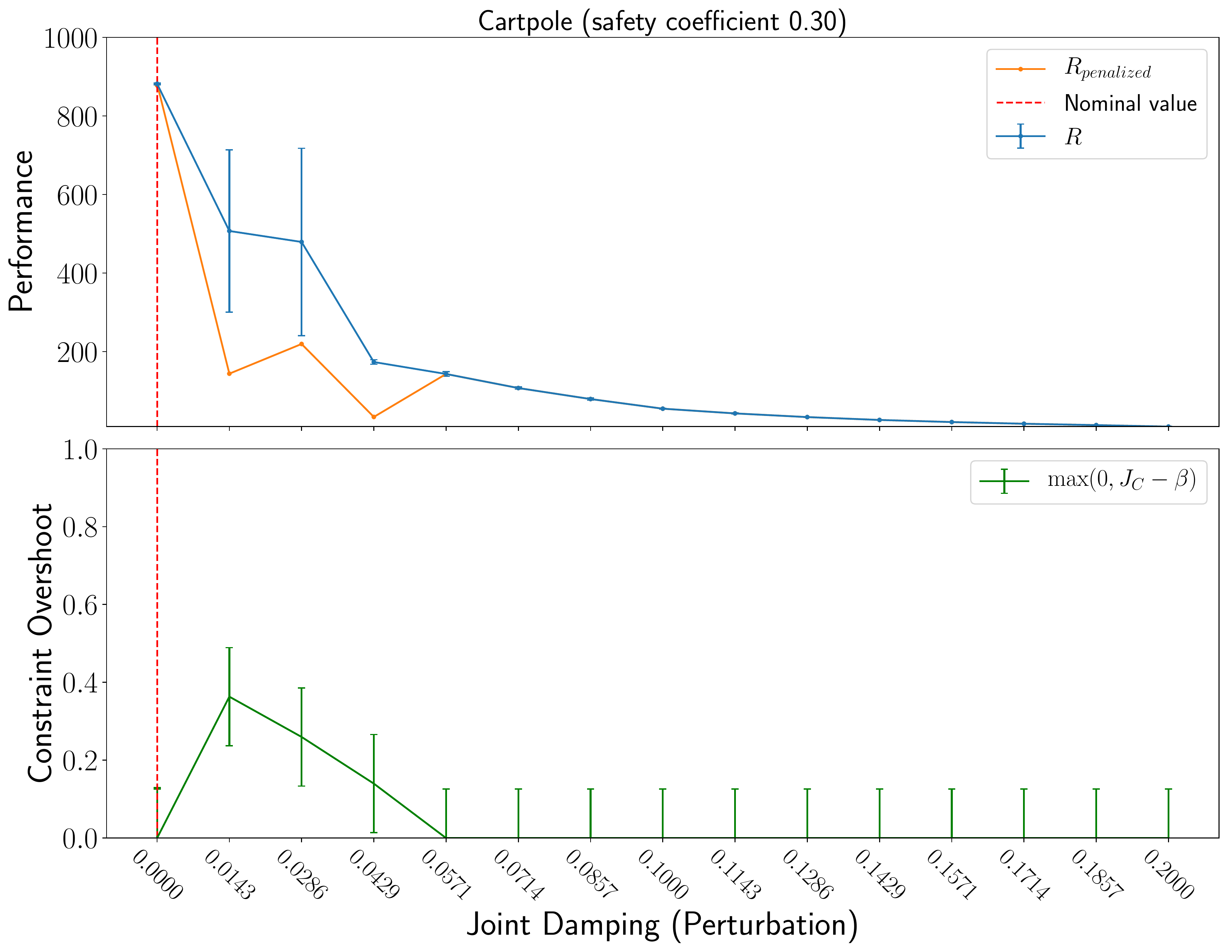}
        \includegraphics[width=0.48\textwidth]{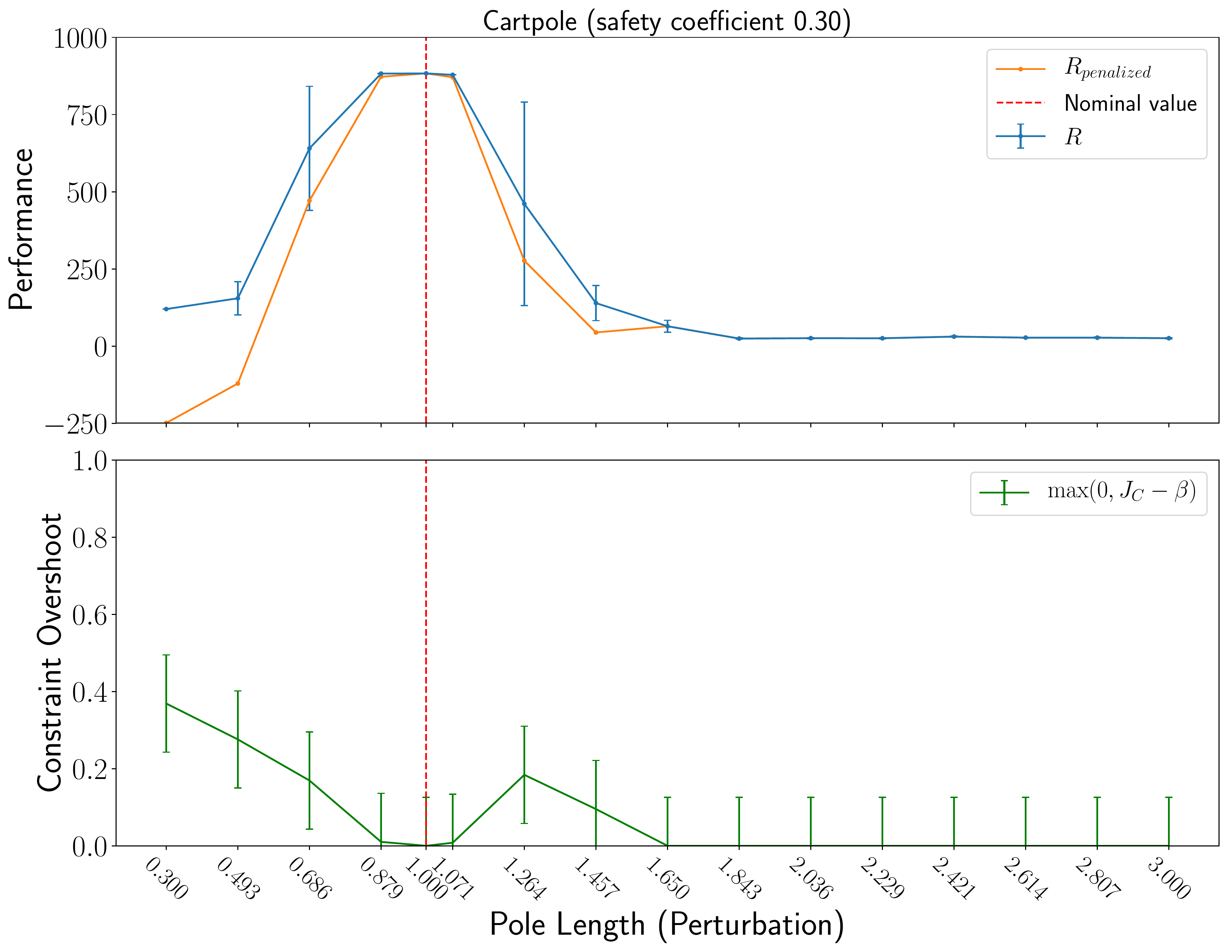}
        \includegraphics[width=0.48\textwidth]{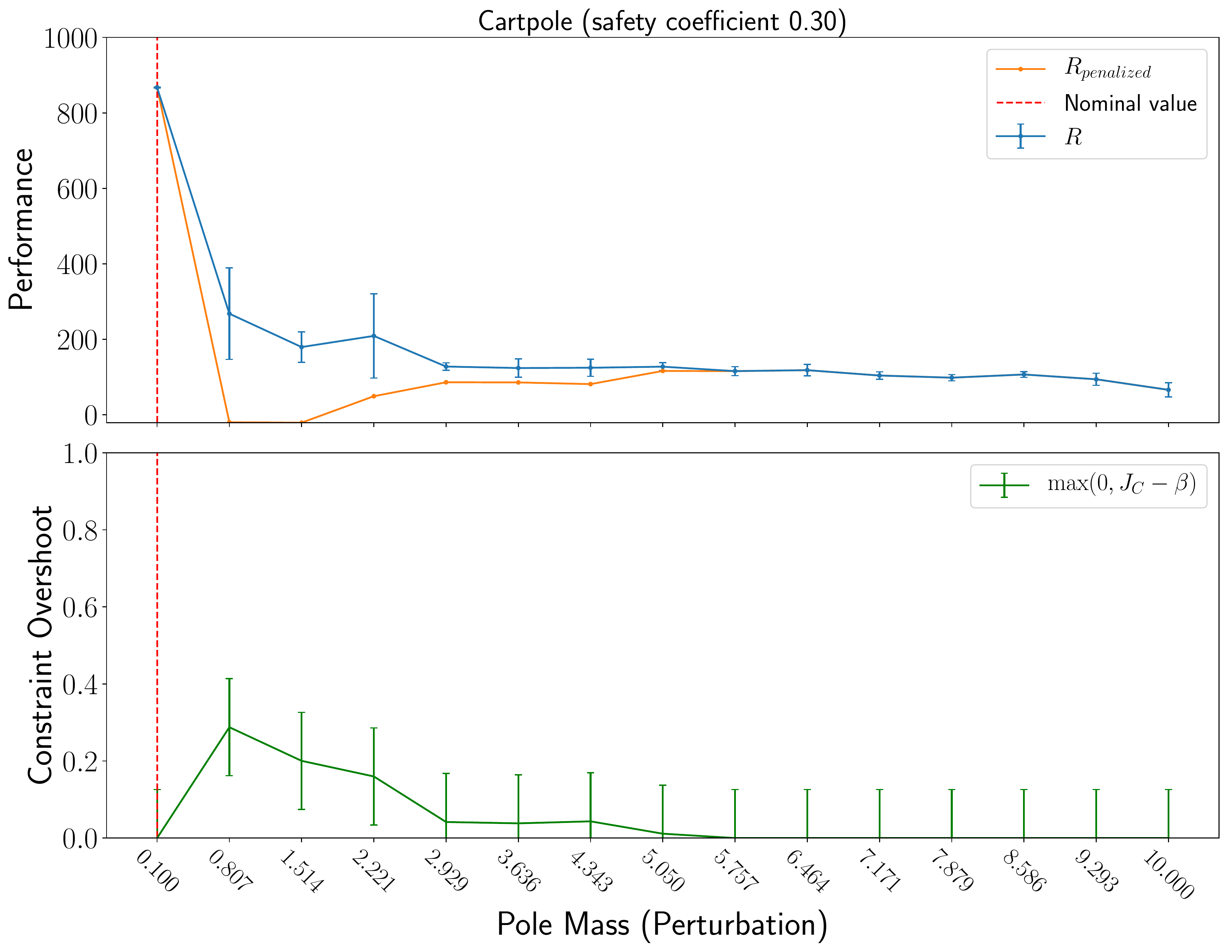}
        \includegraphics[width=0.48\textwidth]{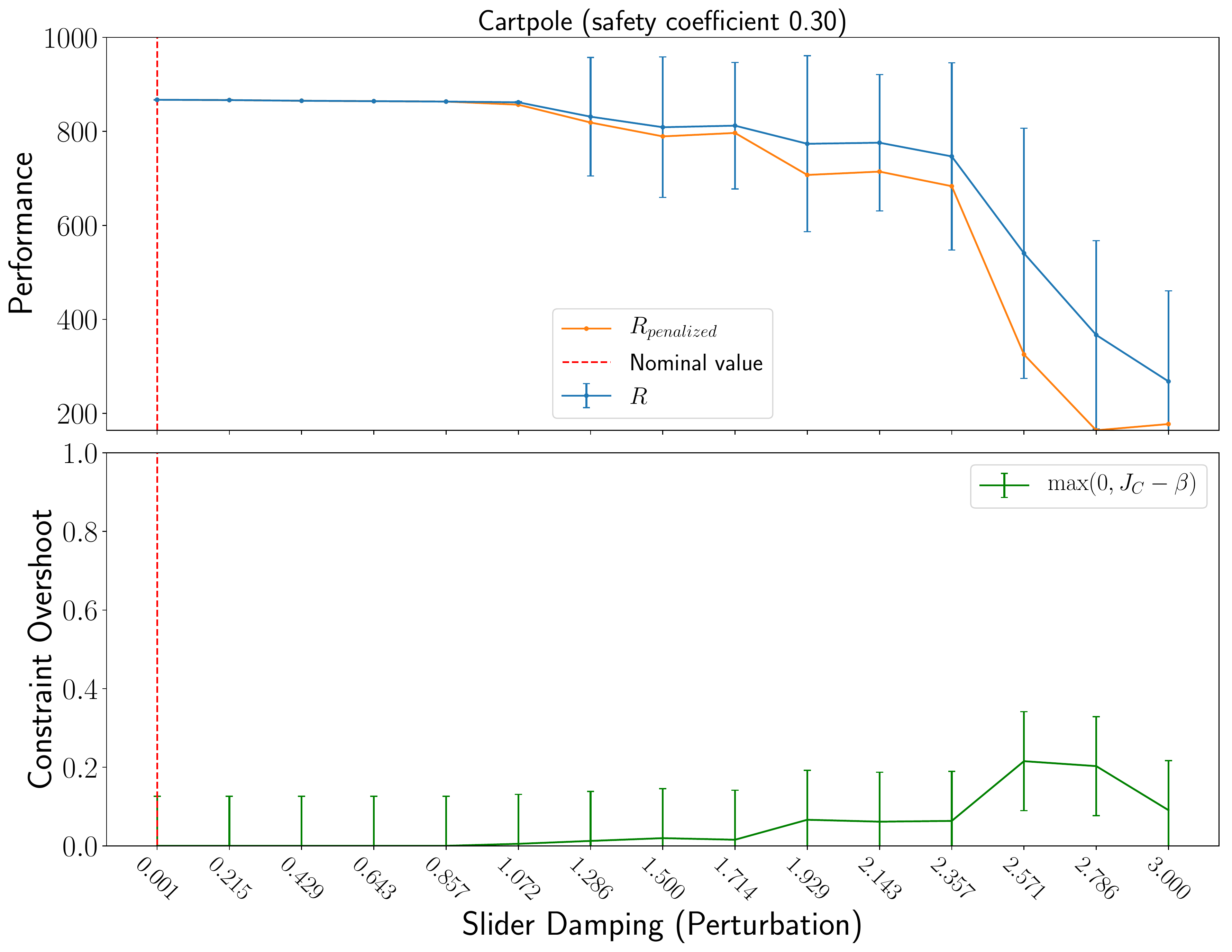}
    \caption{The effect on constraint satisfaction and return as perturbations are added to \tasktext{cartpole} for a fixed C-D4PG policy.}
    \label{fig:sensitivity_cartpole}    
\end{figure}

\begin{figure}[!h]
    \centering
        \includegraphics[width=0.48\textwidth]{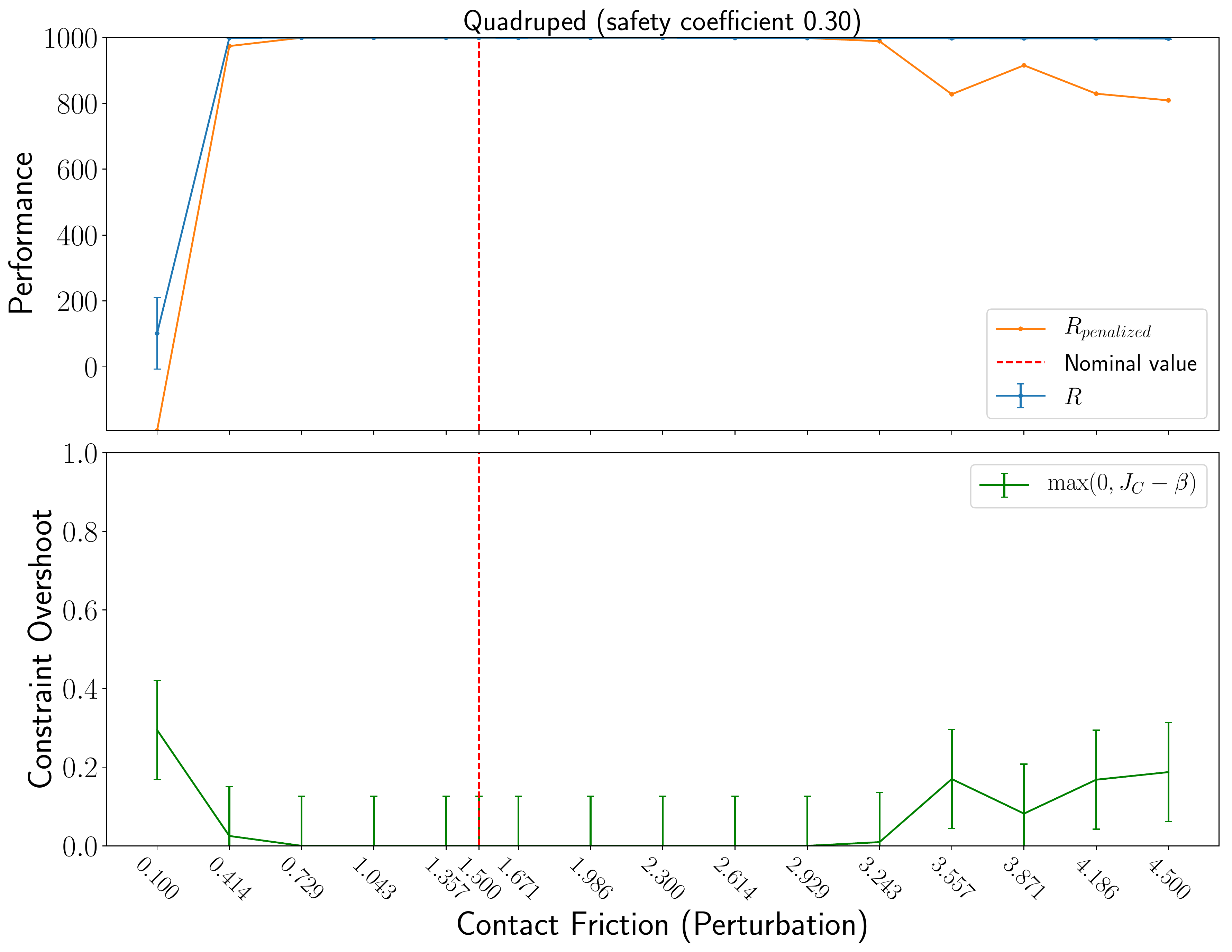}
        \includegraphics[width=0.48\textwidth]{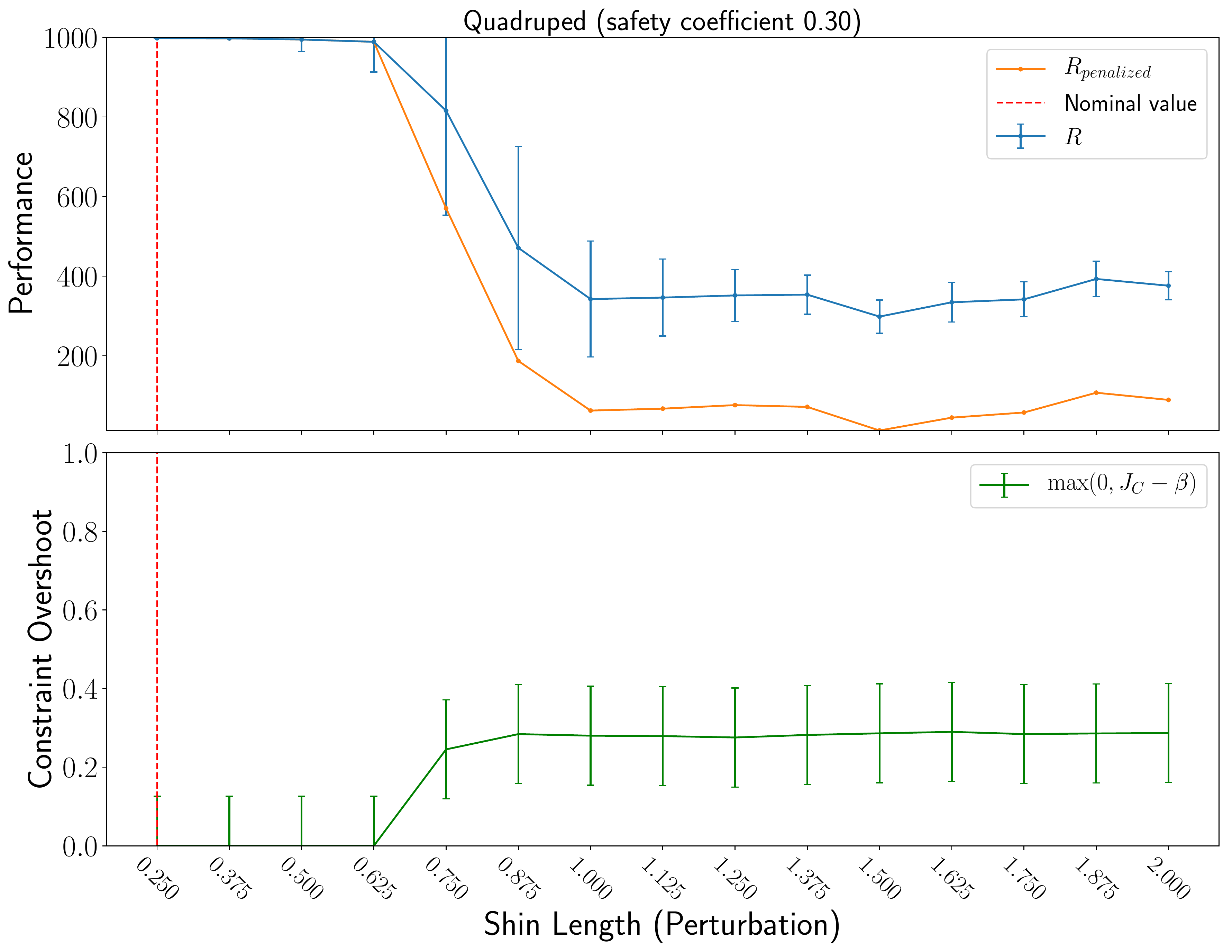}
        \includegraphics[width=0.48\textwidth]{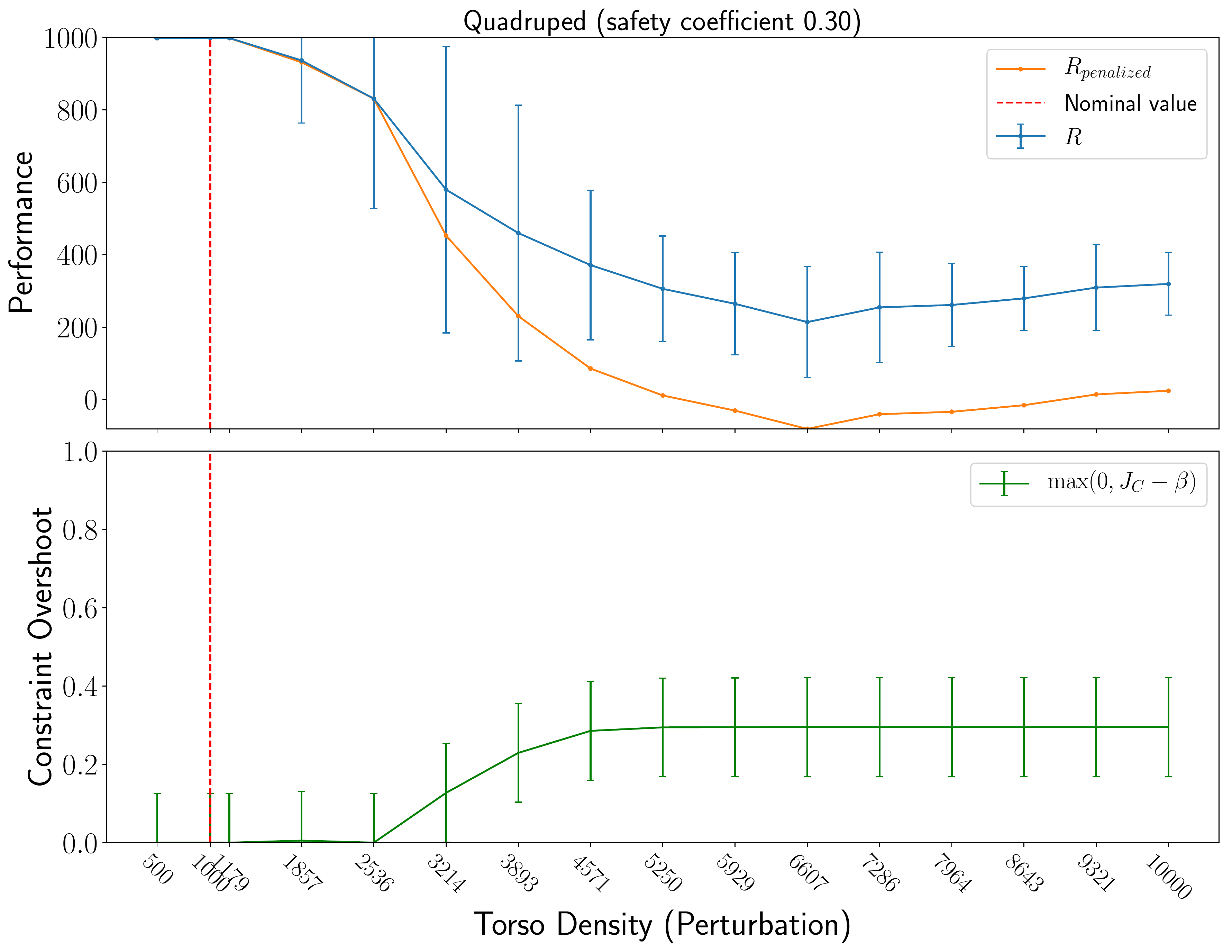}
    \caption{The effect on constraint satisfaction and return as perturbations are added to \tasktext{quadruped} for a fixed C-D4PG policy.}
    \label{fig:sensitivity_quadruped}    
\end{figure}

\begin{figure}[!h]
    \centering
        \includegraphics[width=0.48\textwidth]{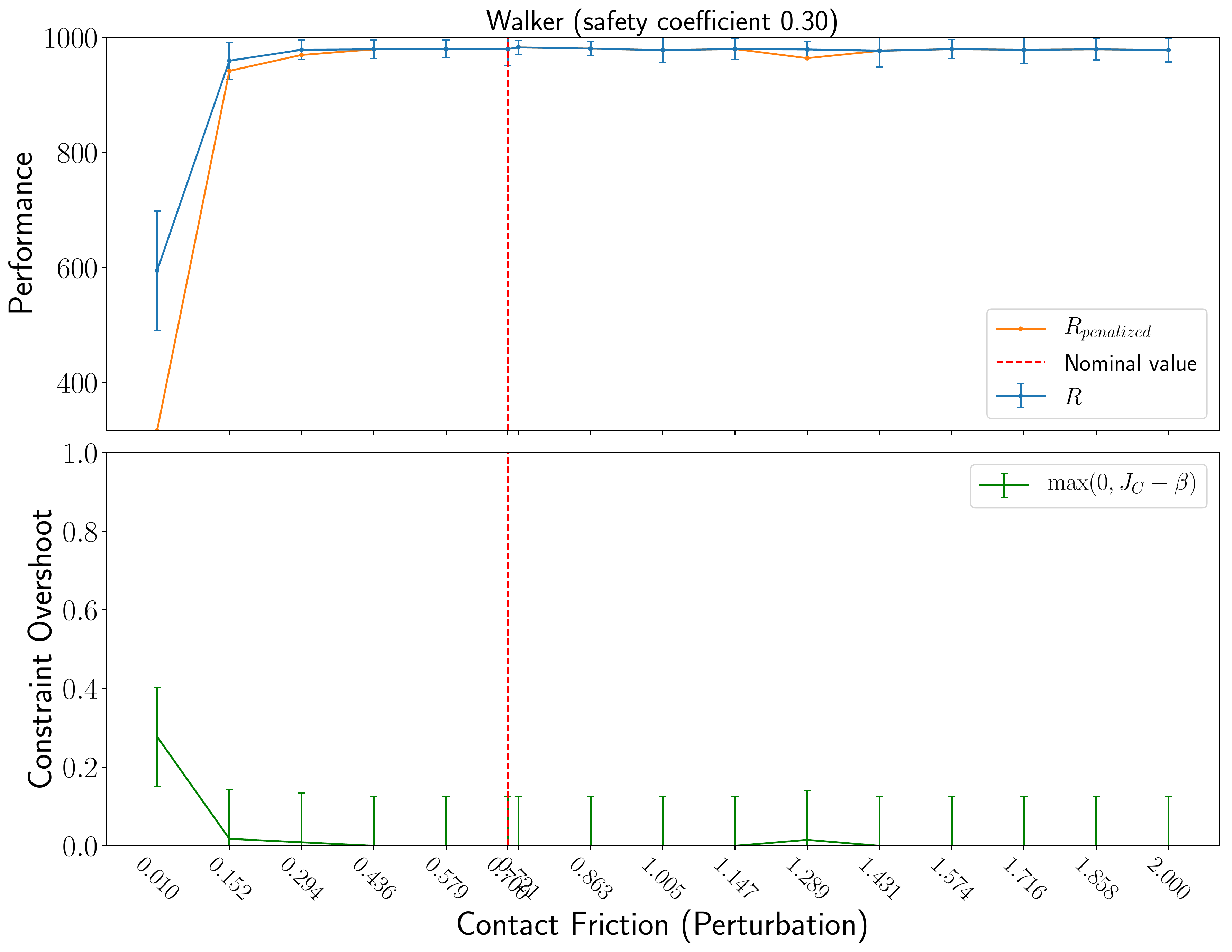}
        \includegraphics[width=0.48\textwidth]{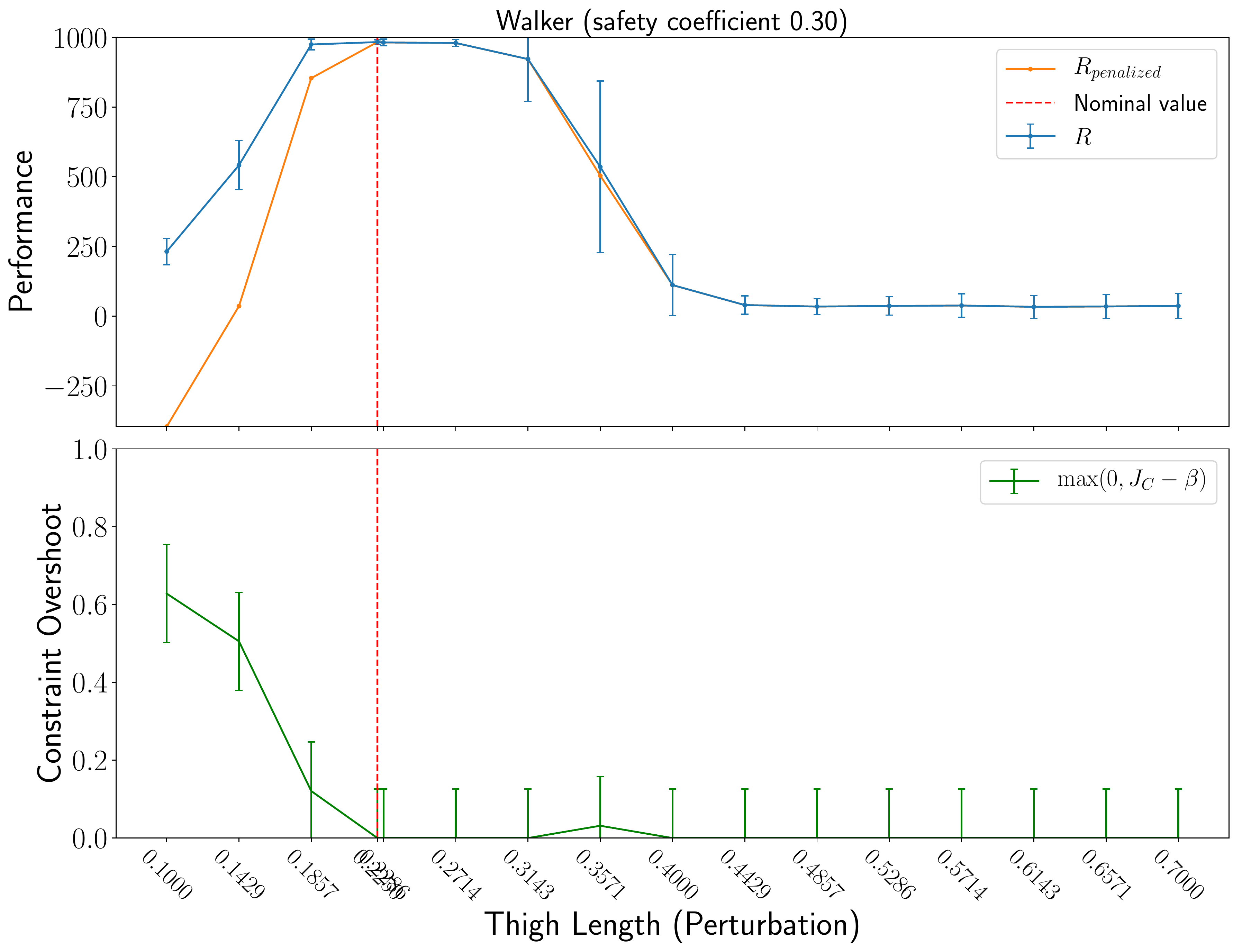}
        \includegraphics[width=0.48\textwidth]{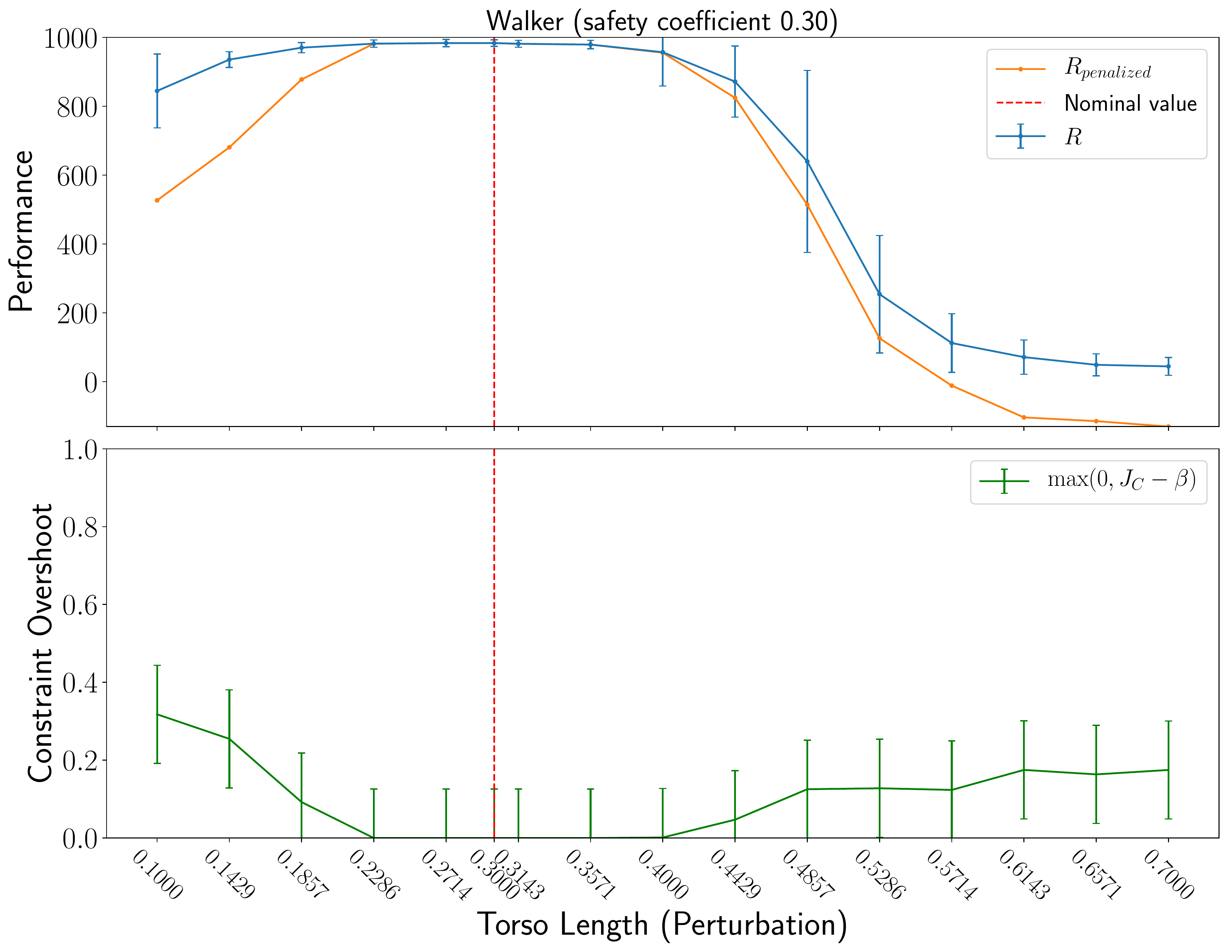}
    \caption{The effect on constraint satisfaction and return as perturbations are added to \tasktext{walker} for a fixed C-D4PG policy.}
    \label{fig:sensitivity_walker}    
\end{figure}

\subsection{Robustness performance of D4PG and DMPO variants}

\begin{figure}[!h]
    \centering
        \includegraphics[width=0.83\textwidth]{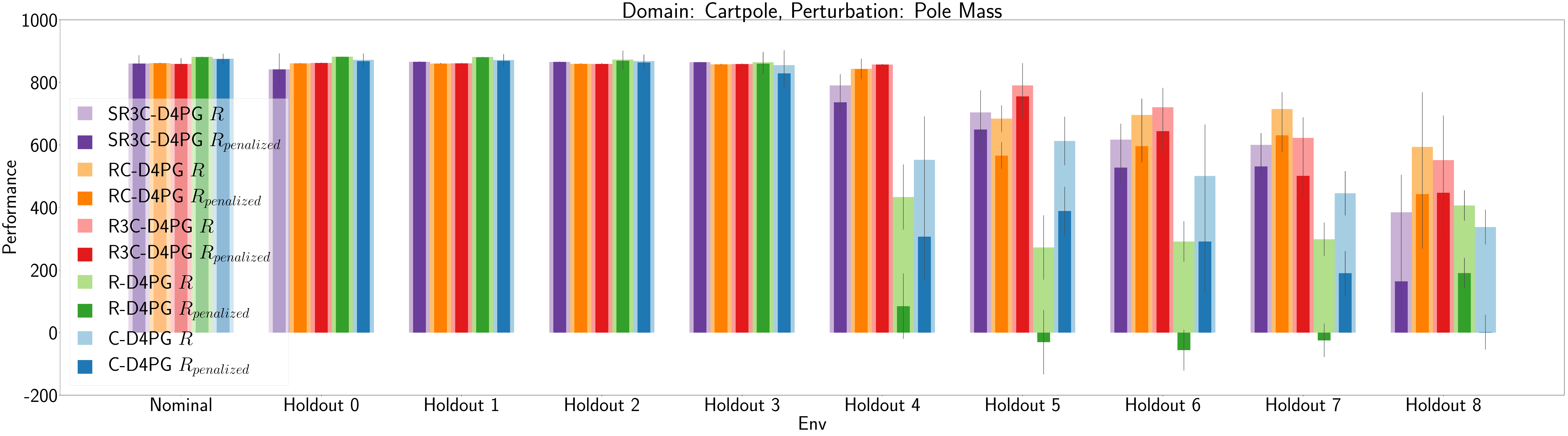}
        \includegraphics[width=0.83\textwidth]{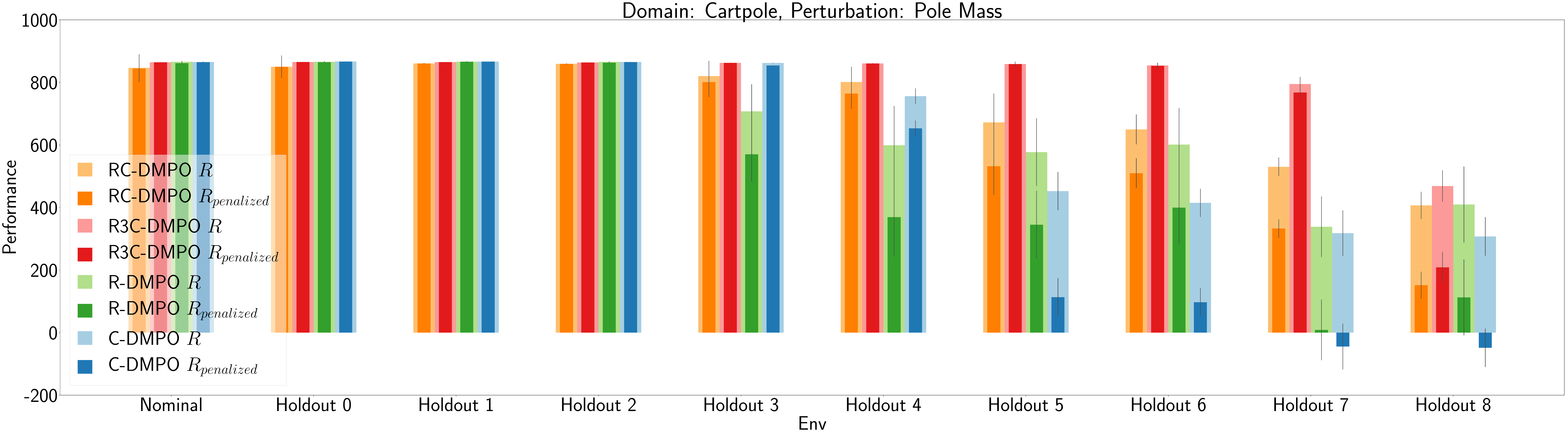}
        \includegraphics[width=0.83\textwidth]{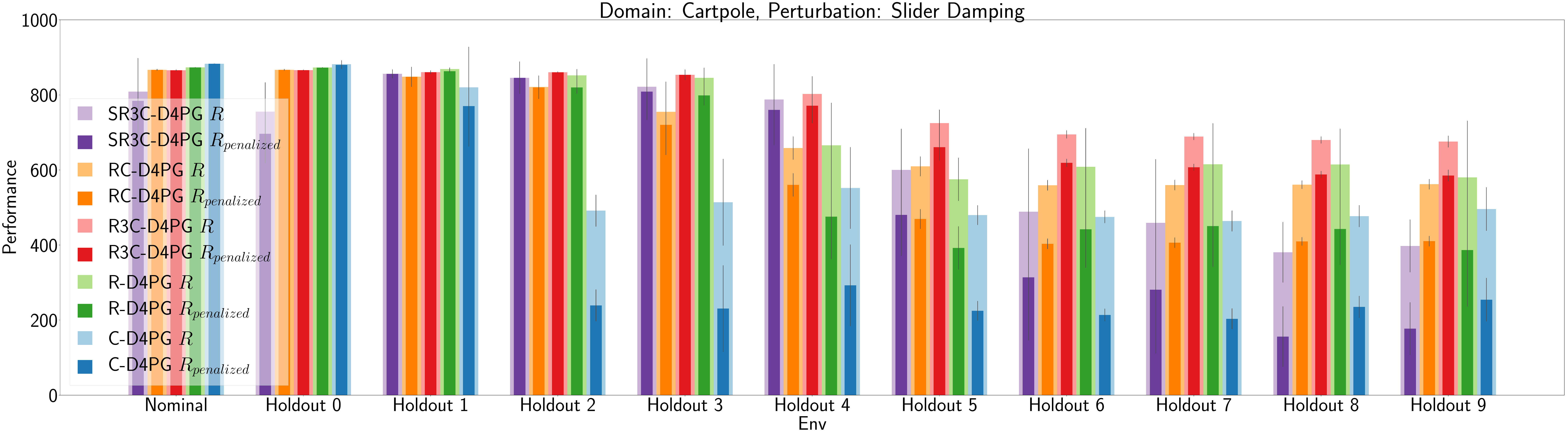}
        \includegraphics[width=0.83\textwidth]{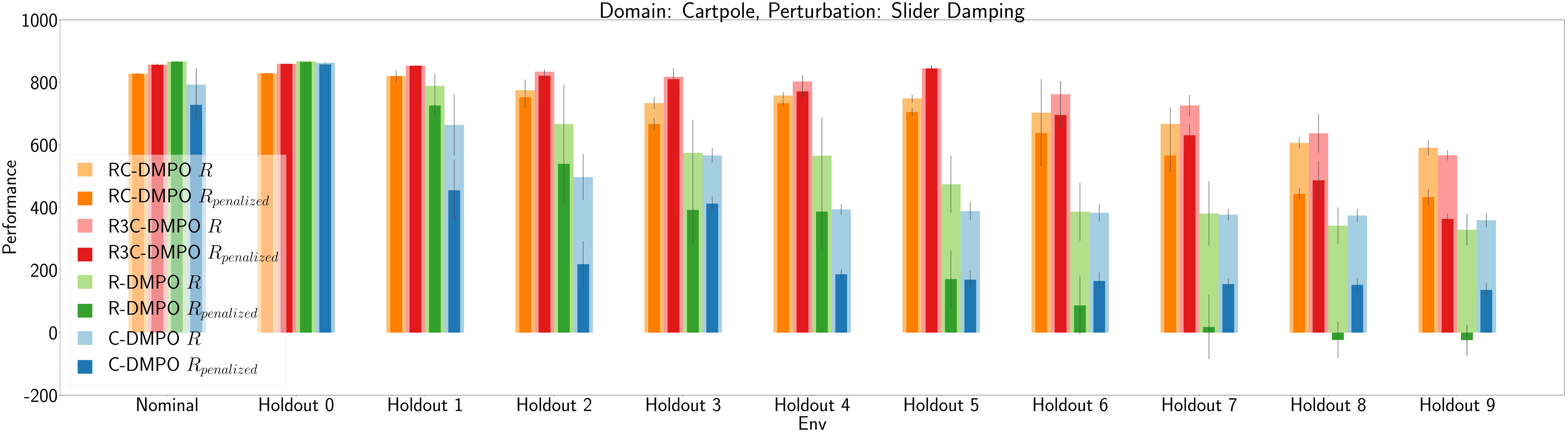}
        \includegraphics[width=0.83\textwidth]{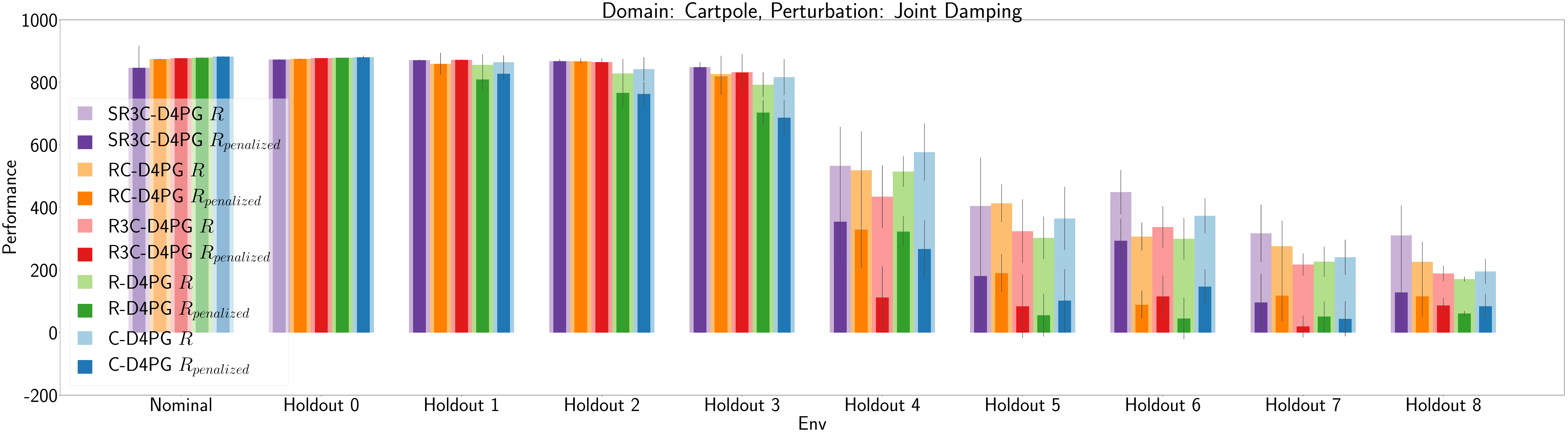}
        \includegraphics[width=0.83\textwidth]{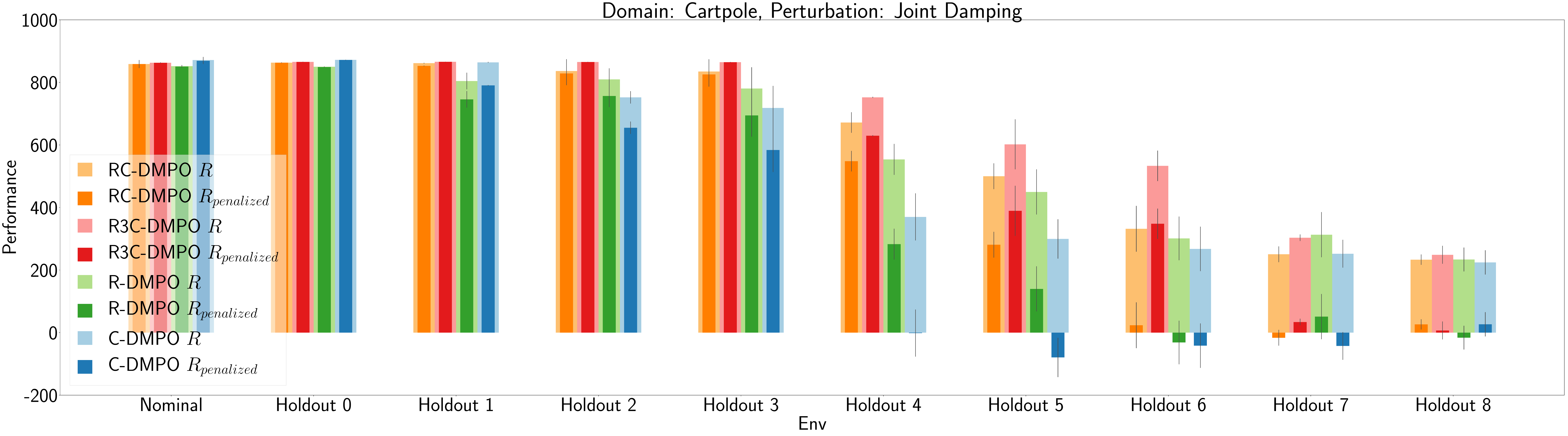}
    \caption{The robustness performance of the D4PG variants per task (row).}
    \label{fig:overall_results1}    
\end{figure}

\begin{figure}[!h]
    \centering
        \includegraphics[width=0.83\textwidth]{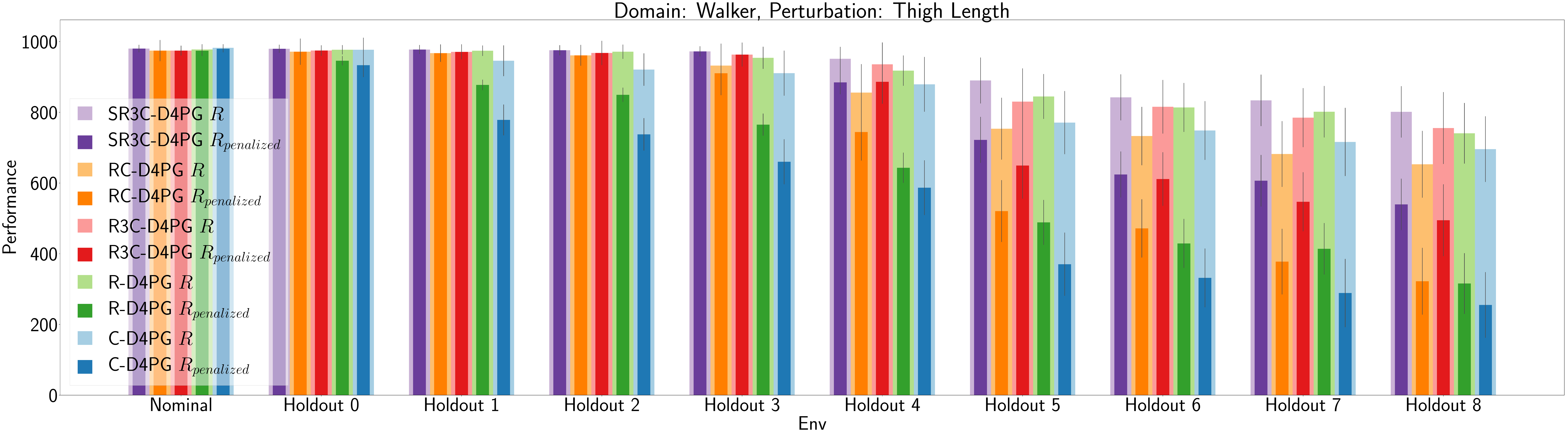}
        \includegraphics[width=0.83\textwidth]{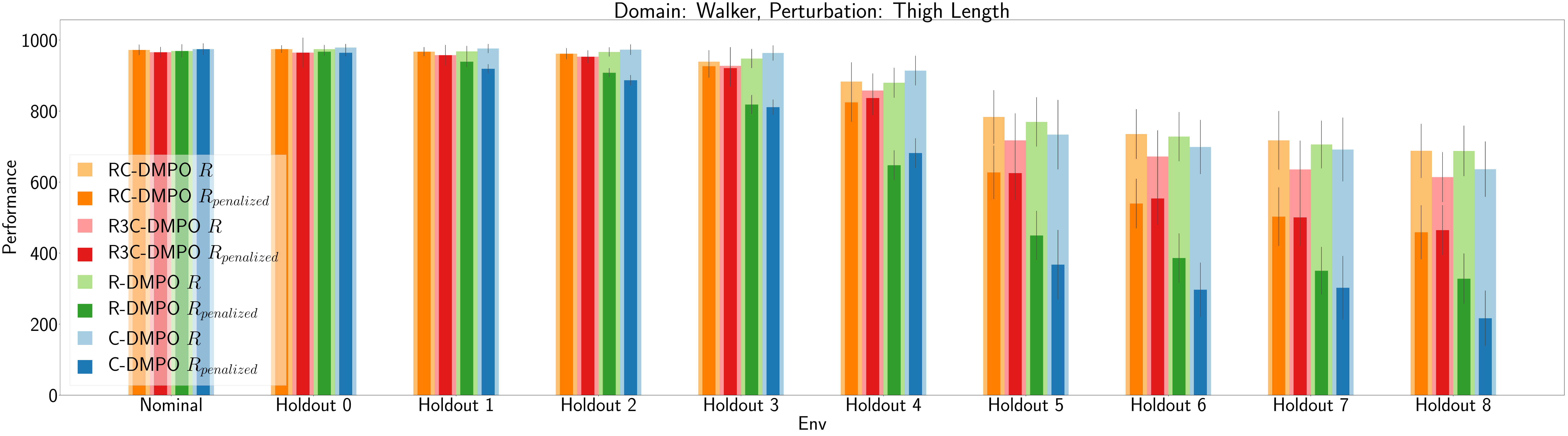}        
        \includegraphics[width=0.83\textwidth]{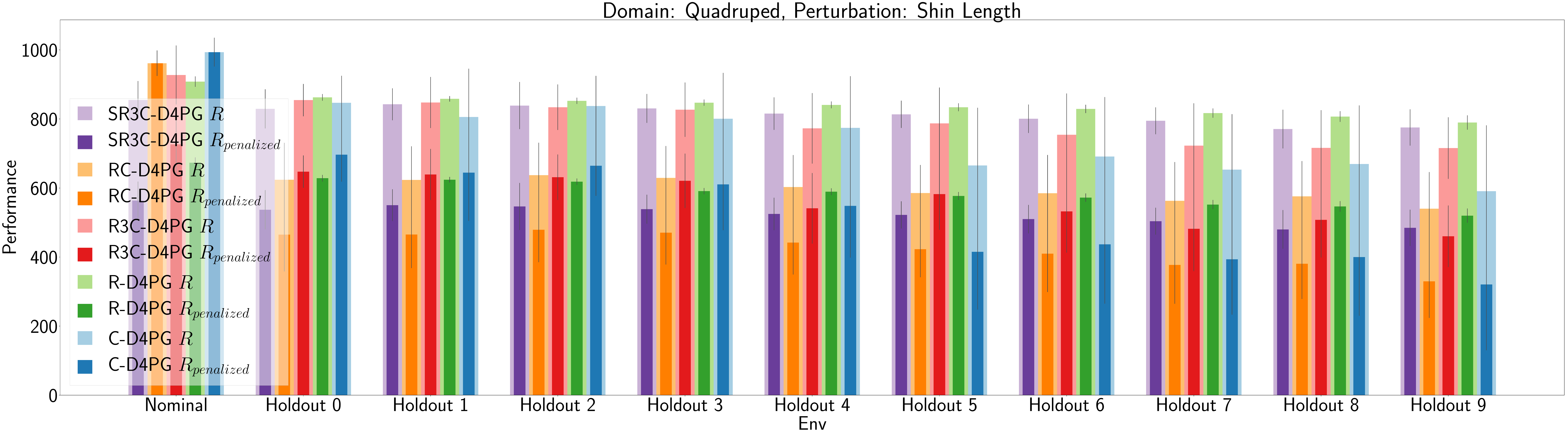}
        \includegraphics[width=0.83\textwidth]{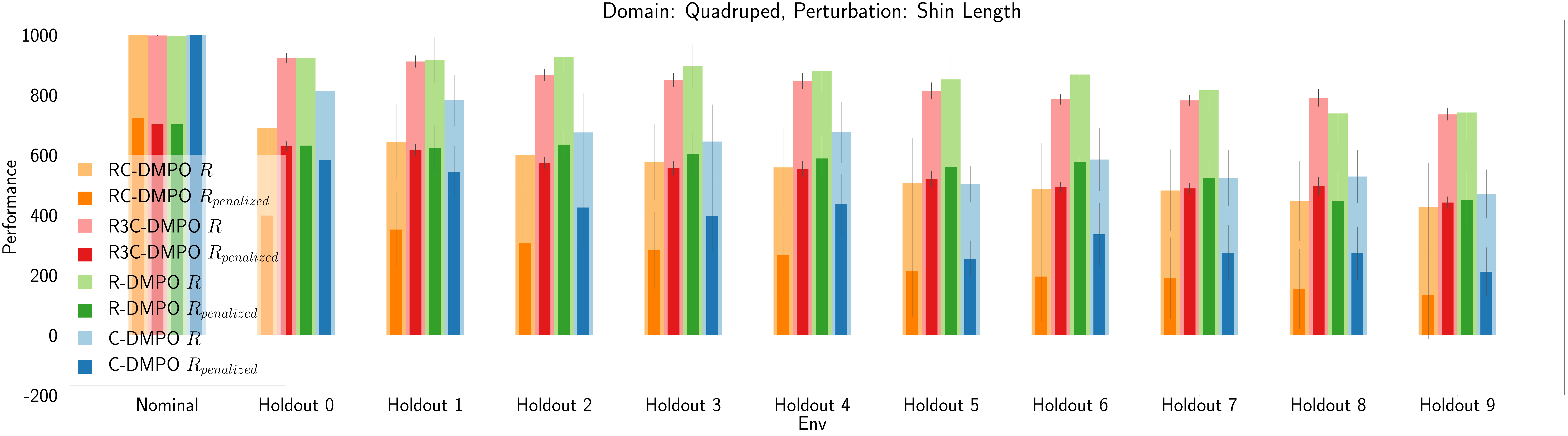}
        \includegraphics[width=0.83\textwidth]{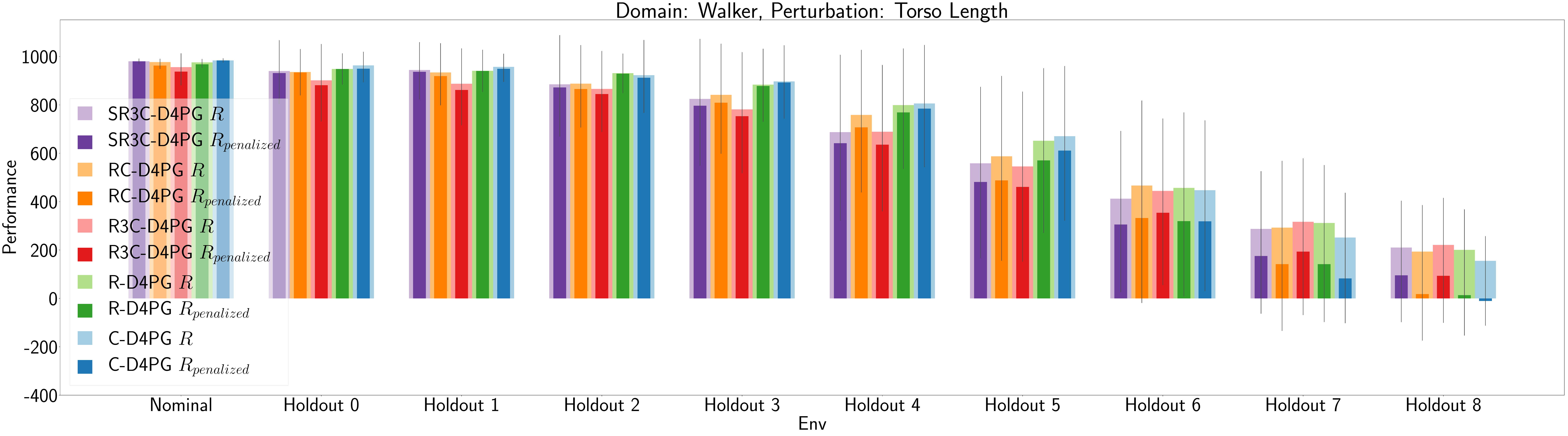}
        \includegraphics[width=0.83\textwidth]{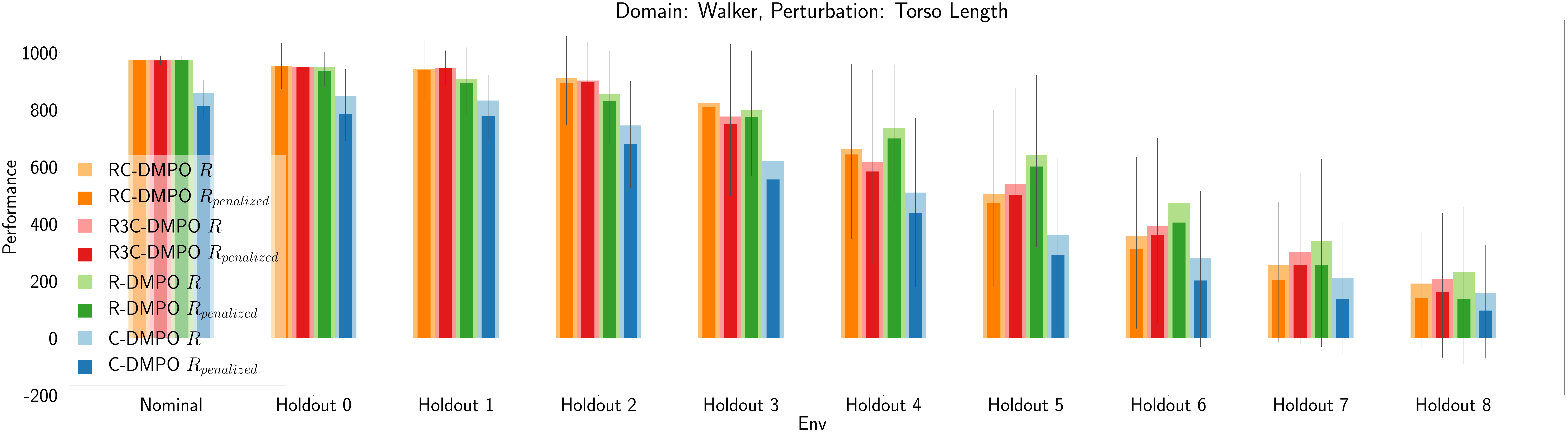}
    \caption{The robustness performance of the DMPO variants per task (row).}
    \label{fig:overall_results2}    
\end{figure}

\subsection{Investigative Studies}
\label{app:investigate}

\begin{figure}[!h]
    \centering
        \includegraphics[width=0.9\textwidth]{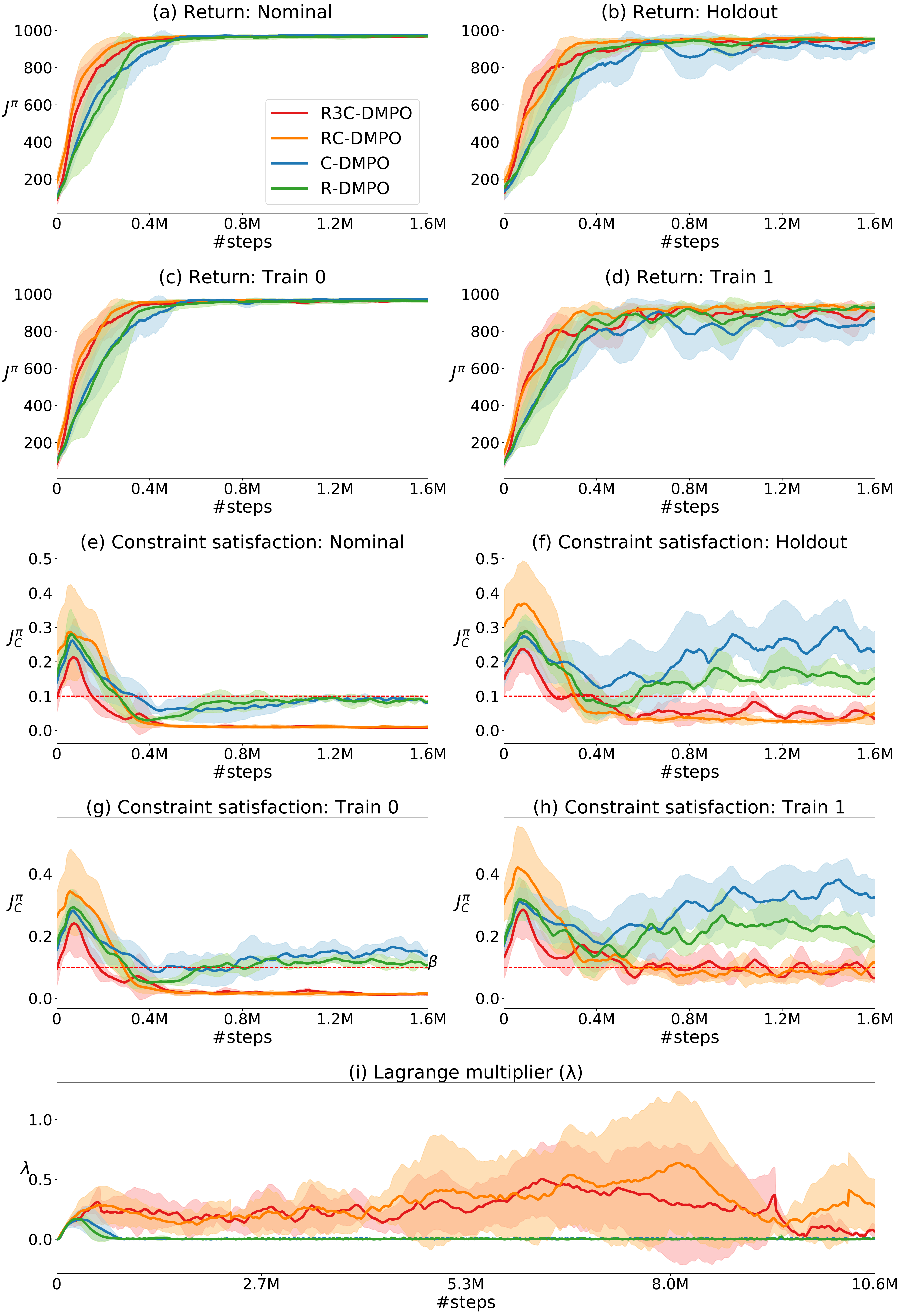}
    \caption{Extended variant of Figure 4 from the main paper.}
    \label{fig:investigate}    
\end{figure}

\end{document}